\documentclass{article}

\usepackage{arxiv}

\usepackage[utf8]{inputenc} 
\usepackage[T1]{fontenc}    
\usepackage{hyperref}       
\usepackage{url}            
\usepackage{booktabs}       
\usepackage{amsfonts}       
\usepackage{nicefrac}       
\usepackage{microtype}      
\usepackage{lipsum}
\usepackage{graphicx}
\graphicspath{ {./images/} }

\usepackage{xcolor}

\usepackage{comment}
\usepackage{amsmath,amssymb} 
\usepackage{color}
\usepackage{multirow}
\usepackage{subfigure}

\usepackage{algorithm}
\usepackage{algorithmic}

\usepackage{tikz-cd}

\newtheorem{theorem}{Theorem}
\newtheorem{property}{Property}
\newtheorem{definition}{Definition}
\newtheorem{corollary}{Corollary}[theorem]
\newtheorem{example}{Example}
\newtheorem{proof}{Proof}

\title{Monotone Boolean Functions, Feasibility/Infeasiblity, LP-Type Problems and MaxCon}

\author{
 David Suter \\
  School of Computing and Security\\
  School of Science\\  Edith Cowan University \\
  \texttt{d.suter `at' ecu with edu.au} \\
   \And
  Ruwan Tennakoon\\
  Computer Science and Software Engineering\\ School of Science\\ RMIT University \\
  \And
  Erchuan Zhang\\
  School of Computing and Security\\
  School of Science\\  Edith Cowan University \\
  \And
  Tat-Jun Chin\\
  School of Computer Science\\
  The University of Adelaide\\
  \And
 Alireza Bab-Hadiashar \\
  School of Engineering\\  RMIT University 
}

\begin{document}
\maketitle

\begin{abstract}
This paper outlines connections between Monotone Boolean Functions, LP-Type problems and the Maximum Consensus  Problem. The latter refers to a particular type of robust fitting characterisation, popular in Computer Vision (MaxCon). Indeed, this is our main motivation but we believe the results of the study of these connections are more widely applicable to LP-type problems (at least ``thresholded versions'', as we describe), and perhaps even more widely.

We illustrate, with examples from Computer Vision, how the resulting perspectives suggest new algorithms.
Indeed, we focus, in the experimental part, on how the Influence (a property of Boolean Functions that takes on a special form if the function is Monotone) can guide a search for the MaxCon solution. 

\end{abstract}

\tableofcontents

\section{Introduction}

The popular Maximum Consensus (MaxCon) criterion for robust fitting (as typified by that of RANSAC \cite{10.1145/358669.358692}), seeks the maximum sized feasible set. 
Here feasible means that all data points belonging to the ``structure'' (the inlier set) fits the model within some tolerance level. We call this ``being within an epsilon-band of a model''.

The set of all possible subsets of the given data, is related to the Boolean Cube.
Thus statements about which of the subsets are feasible (can be contained in tolerance-level defined region around an instance of the model) is a statement of the evaluation of a Boolean Function, (say, which outputs 1 for infeasible, and 0 for feasible) over the Boolean Cube.

This view does not seem to have been specifically emphasised before - though we don't claim no one has observed that view. 
Such a view immediately opens the huge theory and very many associated mathematical tools, developed during the study of Boolean Functions; for the purposes of the analysis of the MaxCon problem and for devising algorithms for solving it.

However, there is one more crucial observation. We are not just talking about any Boolean Function: the inherent properties of feasibility and infeasibility dictate that we are talking about {\em special} Boolean Functions - the Monotone Boolean Functions \cite{Korshunov_2003}. 
This special class of Boolean Functions has received particular attention in the literature.

In this paper we, in a sense, begin the journey, of investigating what the theory of Monotone Boolean Functions has to say about the MaxCon problem and about what algorithms the theory may suggest.
For this purpose we concentrate on a property of all Boolean Functions, called Influence \cite{o'donnell_2014}; a property that has a particular relationship with the Fourier Transform of a Boolean Function {\em when that Boolean Function is Monotone}. 

Before proceeding with details of the paper, we make a slightly tongue-in-cheek remark. Since we show that MaxCon, when viewed in the above sketched framework, is nothing more than the search for the maximum zero (see definition and explanation below) of the above-mentioned Monotone Infeasibility Function, and since Monotone Boolean Functions have been studied for decades, it could be said that MaxCon was studied even before it was defined. 
Indeed, as long ago as the 1970's, the question of how to search the Boolean Cube for the maximum zero of a Monotone Boolean Function was studied \cite{KULYANOV1975267}.
This abstract problem includes MaxCon as a special case (the special case where the Monotone Boolean Function in question is the one that returns, for a given input subset of the data, the feasibility or infeasibility of that subset being contained in an epsilon band around any model).
However, the concepts and tools we take advantage of in this paper, are more recent - though some of them have still been in the literature for a considerable time.

Whilst Monotone Boolean Functions (MBFs) have been extensively studied for a variety of application domains,
including learning theory (and computer vision is, these days, highly dominated by learning style approaches),
there appears to be relatively little attention to MBFs in Computer Vision: and more specifically in model-based computer vision. Reference \cite{RAMALINGAM20171} appears to be a recent exception - but even this is tackling  very different considerations from our main areas of interest (namely, \cite{RAMALINGAM20171} is concerned with efficient Conditional Random Field (CRF) calculations, and CRF modelling is very different to the geometric modelling we refer to).

Lastly, just as we now know that MaxCon can be placed within the framework of LP-type problems \cite{7873535}, much of what we present in this paper, has analogues in that broader class of problems.
Indeed, we highlight that since Monotone (though not necessarily Boolean) functions  are an inherent property, of LP-Type problems, there is a clear link to Monotone Boolean Functions (through thresholding - if that function in question is not already Boolean). Of course, we have to consider not just the range of the function, but also the domain: and for an arbitrary LP-type problem the domain of the Monotone function  is not always the Boolean Cube. We leave exploring how the ideas presented here need to be elaborated for more general domains (e.g., lattices) for further work.


We conclude this introduction by giving some concrete examples of LP-type problems: see Figure \ref{fig:fig}, and sketching some terminology and concepts.
\begin{figure}
  {  \centering
\subfigure[]{
  \includegraphics[width=.5\linewidth]{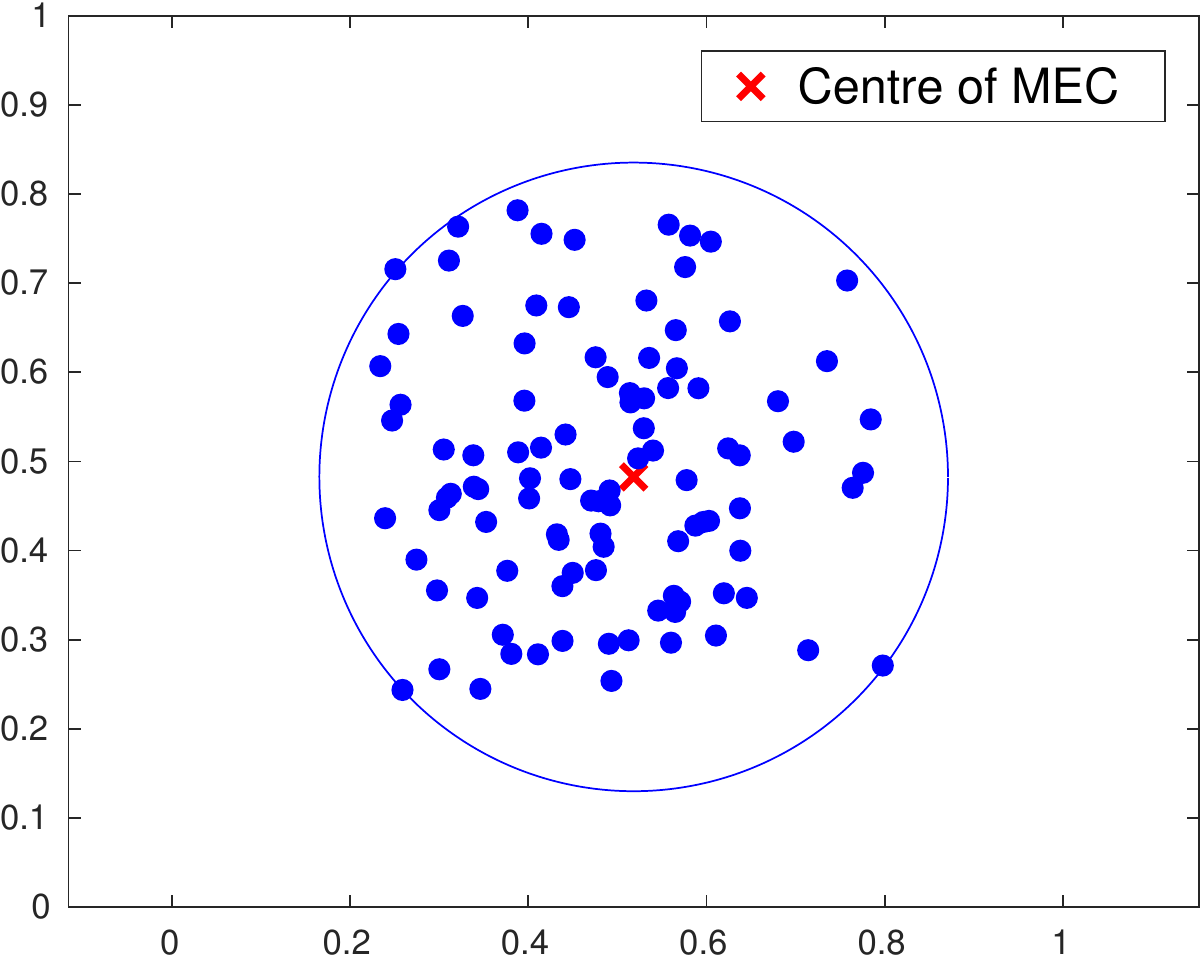}  
  \label{fig:sub-first}
}
\subfigure[]{
  \includegraphics[width=.5\linewidth]{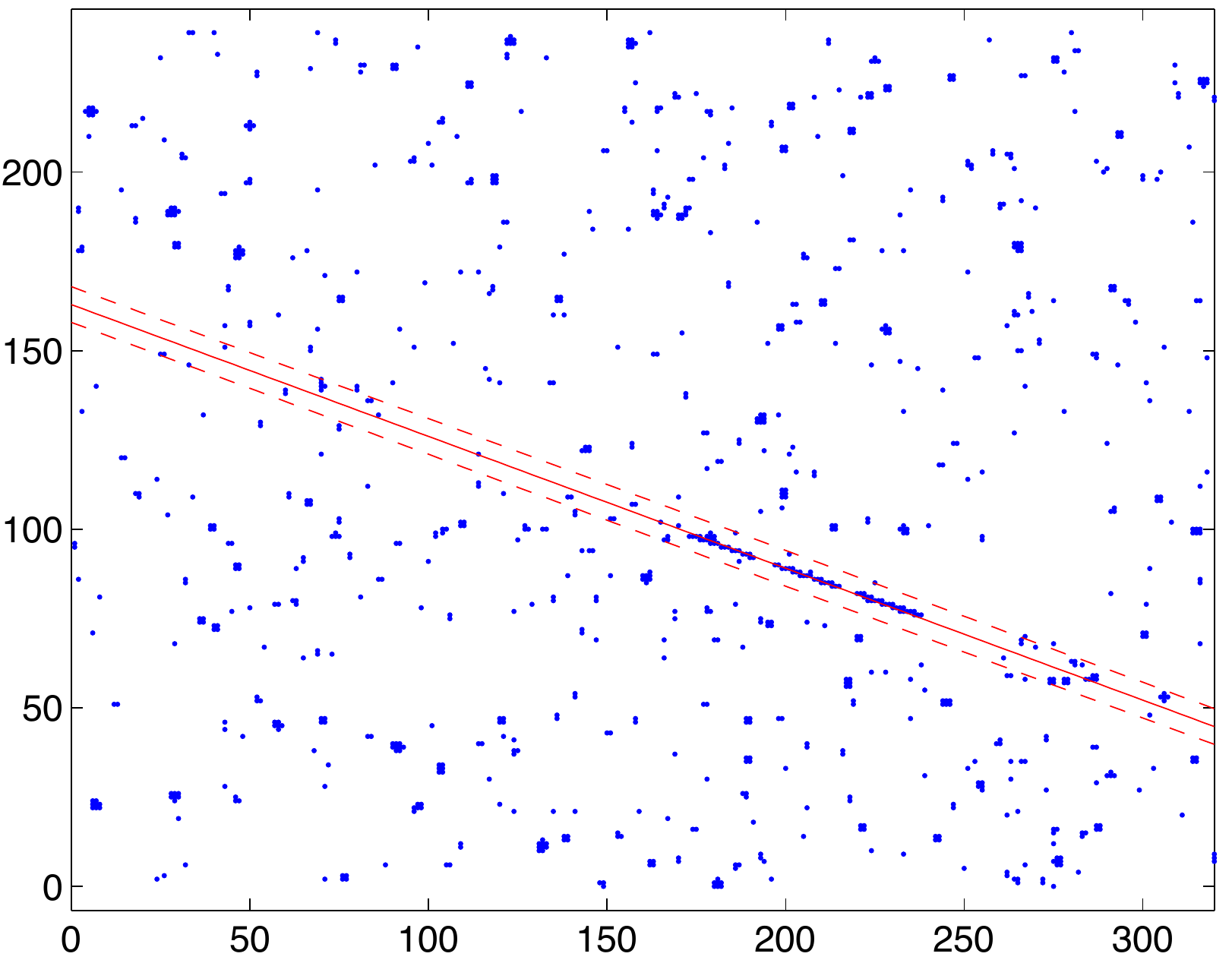}  
  \label{fig:sub-second}
}}
\caption{Illustrative Examples: Problem {\em instances} of MEC and Linear Slab}
\label{fig:fig}
(Left) Given a large number of data points, in some high dimension (here only 2 for illustration) - find the minimum enclosing ball (circle in 2-dim, so called Minimum Enclosing Circle problem). (Right) Given a FIXED ``slab width'' find the position of the slab that encloses the most points. This example is a sattelite track found amongst stars \cite{7873535}.
These are {\em not} indicative of the breadth of problems, nor of the complexity of problems. They are merely chosen to be easy to understand, concrete, and of use in explanation below.
\end{figure}

At the heart of our motivation is (robust) model fitting. The model could be that of a familiar geometric object (as in Figure \ref{fig:fig}, a circle or a line). Or it could be some other mathematical model in higher dimensions and with less familiar geometry - in the main, we are interested in computer vision type applications (see sections \ref{FM-est} \ref{motion-est} for a couple of examples).

We think of instances of the models as {\em ``structures''}. A line, a circle, a cylinder, a plane, are all familiar such structures. So in some general sense, we are concerned with finding ``structures'' in data. What makes the task difficult (compared with the obvious situation where a least squared fit of the model to the data would suffice) is that there may be more than one structure in the data, and there may be {\em ``outliers''} (data that belong no structure - or no structure of any interest). 

\begin{figure}
   \centering

  \includegraphics[width=.9\linewidth]{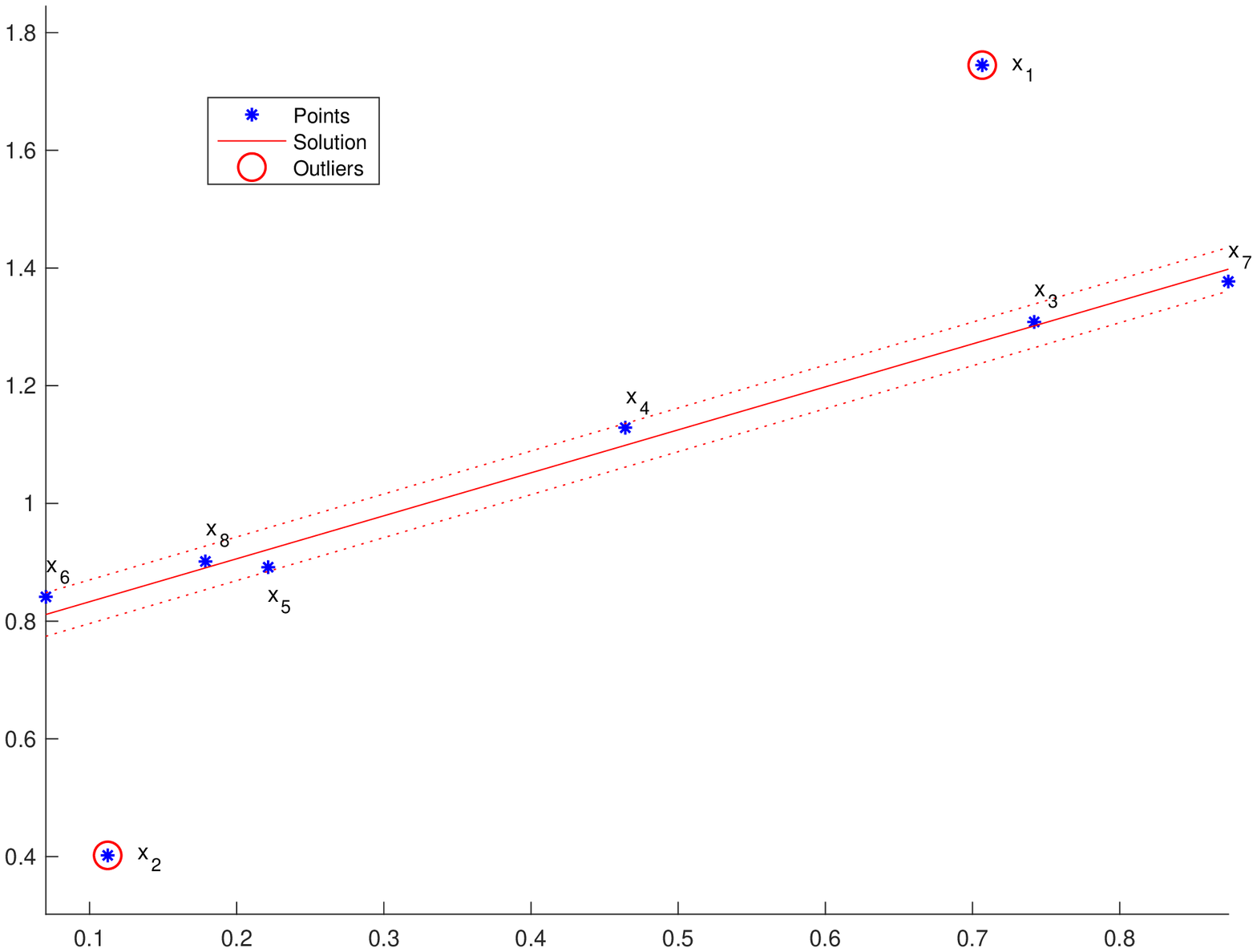}  
\caption{Line Fitting - Simple Example}
\label{linesimple:fig}
8 points in a plane and a MaxCon solution.

\end{figure}

Taking line fitting as our running example: any two points lie on a line, so for us to declare structure, we require three or more points. However, with real data, any three points are unlikely to lie {\em exactly} on a line. So we must allow some tolerance - effectively line fitting becomes ``slab'' fitting.

Figure \ref{linesimple:fig} shows 8 points in a plane. With a certain tolerance, only 6 of the points can be fitted to a line (so that is the MaxCon solution, leaving 2 outliers).
If we were given the data, and plotted the points (not knowing the line), this is the solution our visual system would suggest. One single structure and two outliers. However, one can see that there are other subsets of that data that fit to a line with the same tolerance (and have more than 2 points): for example $\{1,5,6,8\}$, $\{2,3,7\}$  and $\{1,2,4\}$ which involve what we might call (using the first interpretation) ``accidental alignment of outliers with some of the inliers''.
Of course, that interpretation (that these are only accidental) is open to challenge. MaxCon doesn't concern itself with such - it only seeks the largest possible structure and (implicitly) one is accepting that as an interpretation of the data.

Note, the aforementioned list of subsets of points that are larger in size than 2, fitting a line with tolerance, omitted any subsets completely contained in any of the other listed subsets. For example, any subset of size 3 or more taken from the 6 inliers to the line - they will also obviously fit that same line. We would consider these as redundant to the bigger structure interpretation.

We will fully analyse this example later - and we will see that the four mentioned (the MaxCon solution and the three others that are smaller but incorporate an outlier to the MaxCon solution) are the only 4 possible solutions. That is, we will show that these four, in a combinatorial sense, fully describe the MaxCon landscape for this set of data: only subsets (which includes the sets themselves) of size 3 or more, of these four sets, are feasible solutions - other than trivial solutions (of size 2 or less) or redundant solutions (subsidiary to larger structures) are feasible (fit within tolerance to a line).

Of course, we might have situations where the evidence of more than one structure (rather than the accidental alignment type interpretation) appears realistic/desirable. That is, we might expect, or readly believe, that the data contains more than one meaningful structure: we illustrate that next.

Figure \ref{realise:fig} shows seven points. Also shown are some possible linear structures (at some set tolerance). On the left of the figure, we use three linear structures to explain the data. The red encompasses 4 points, the blue and the green both 5 points.
However, There are other ``explanations'' of the data - including the example on the right of the figure that uses only two structures (indeed the same red and the green as before). By Occam's razor we would prefer the second interpretation. Note: with the given tolerance, the reader should be able to convince themselves that no fewer than two structures can be used to cover all points.

\begin{figure}
  {  \centering
\subfigure[]{
  \includegraphics[width=.45\linewidth]{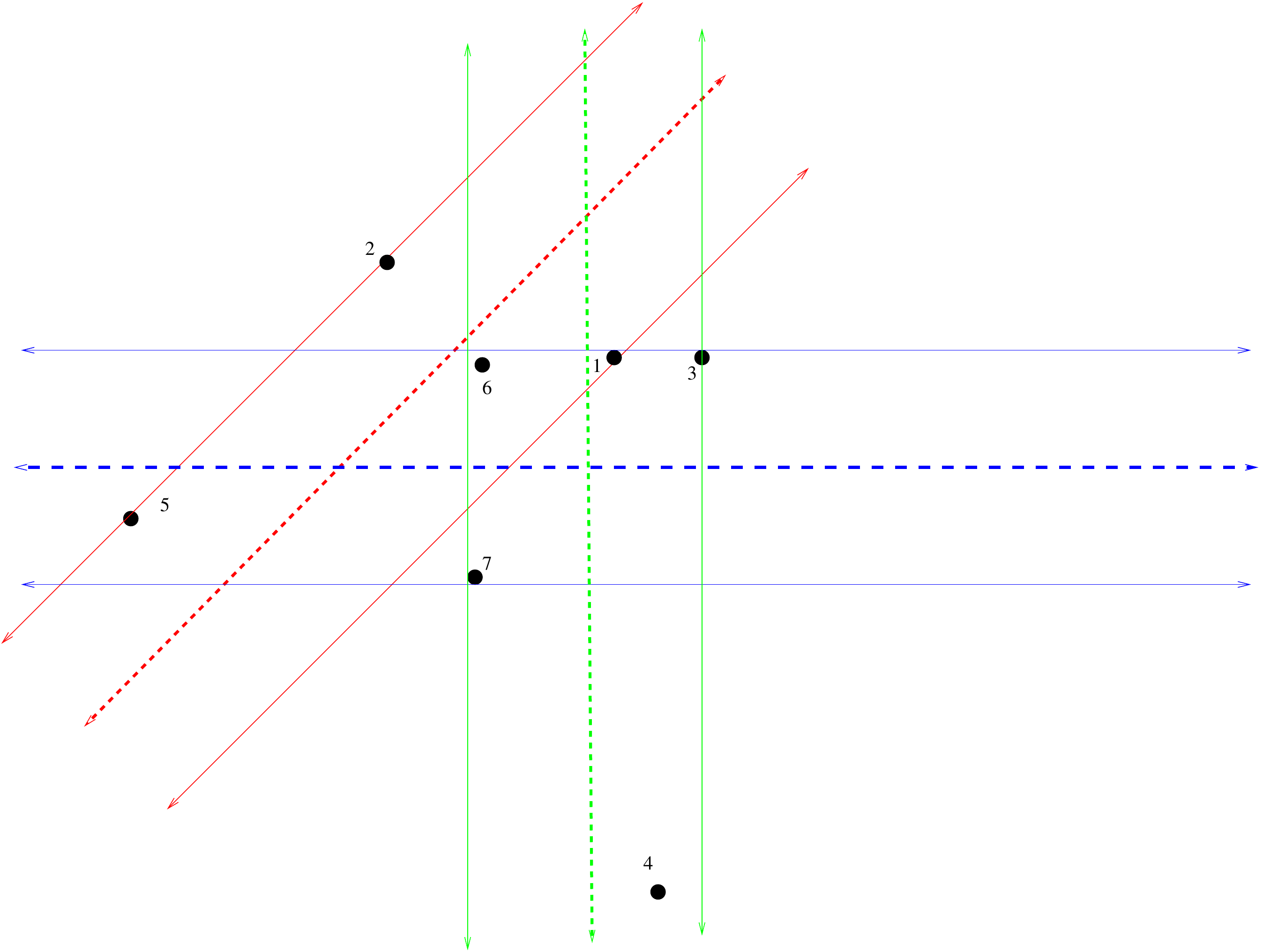}  
  \label{realisefig:sub-first}
}
\subfigure[]{
  \includegraphics[width=.3\linewidth]{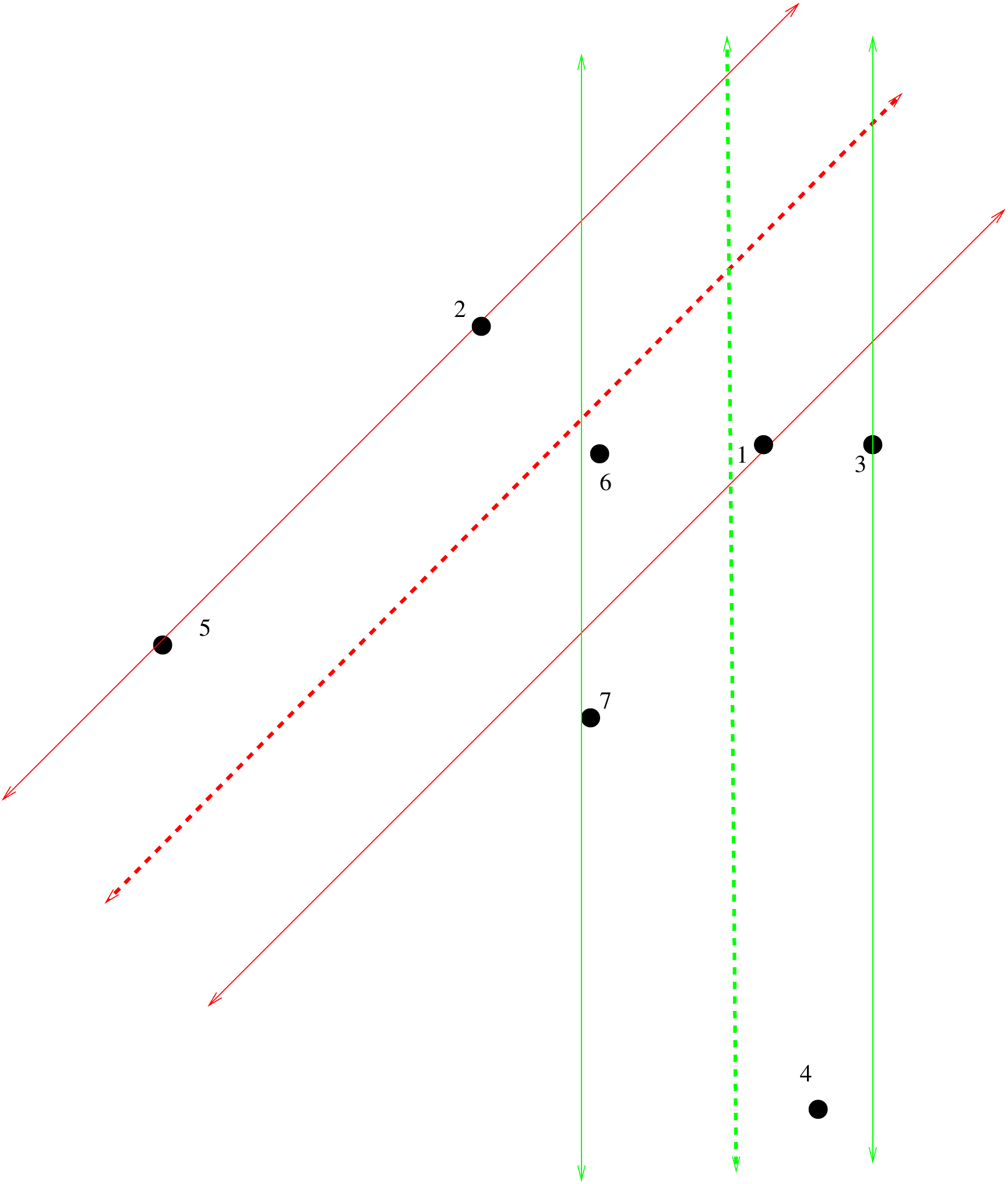}  
  \label{realisefig:sub-second}
}}
\caption{Concept of Structure In Line Fitting}
\label{realise:fig}
(Left) 7 points in a plane and interpretation of structure (line with tolerance, i.e. linear slab enclosing) using 3 lines. (Right) same points with only two lines.

\end{figure}

In the previous example we saw how one linear structure was redundant - the points it explained are already explained by the other two structures (put another way, there are no points in the redundant blue structure that are not in the intersection of that structure with another). If our purpose was to "fit the minimum number of structures that cover the data'' - a multi-structure fitting problem, then the solution depicted in Figure \ref{realisefig:sub-second} would seem to be a good solution. If our purpose was to fit only the largest structure (the one that covers the most points), then the same example shows that there are two alternative answers (each with 5 structures) and our ``reasonable multi-fitting solution'' would have discarded one of them(!). 
MaxCon is inherently a single structure fitting formulation (though it has been adapted for multi-structure recovery) and therefore would not {\em a priori} prefer either the blue (that the multi-structure reasonably discarded) over the green.

In this paper we take only the single structure fitting problem as the goal - albeit possibly in settings where multi-structure interpretations are natural. The only object though, is to recover the largest (by population - note this is correlated with but not necessarily the same as spatial extent) structure is sought.

That said, though the {\em objective} is to recover the largest structure, from a practical point of view one doesn't really care whether the solution returned has a few less data points than the (typically) hundreds or more data actually belonging to the largest structure. Put loosely, we may be willing to trade a little of optimality for substantial gains in speed/reduction in computation.

The above serves as a motivational definition of MaxCon. In effect a large part of the objective of this paper is to give a precise (and abstract) definition of the MaxCon problem. In this abstract definition real world problems (instances of data and the object - e.g., fit a line to to fit as many of the data points as possible) will be mapped to the abstract problem (a Monotone Boolean function over vertices of the Boolean Cube: vertices correspond to subsets of the data and the value of the function correspond to whether the subset can feasibly covered by a line with set tolerance). The abstract problem then becomes one of finding ``maximum upper zeros''. It should be noted that many real word problems map to the same abstract problem (it's a many to one mapping). Moreover, there are abstract problems that it seems (we are sure by looking at examples but have not yet set out to prove) have no realisation - at least for the ``bag of example-types'' we are interested in (line and circle fitting, homography fitting, fundamental matrix fitting etc.).  

\section{Outline}
The paper broadly falls into two parts: section \ref{background-section} is mostly a theoretical sketch of some of the theory we have identified as being relevant, while section \ref{prop-method-section} is mainly experimental - it can be considered as a ``proof'' of concept section where we show the theory has plausible algorithmic implications.

In more detail, section \ref{background-section} has the following elements. We first (section \ref{LP-monotone})draw attention to the link between monotone functions and LP-type problems.

We then (section \ref{MBF-section}) introduce the central character - Boolean-valued Monotone functions over the Boolean Cube. We will drop the qualification ``Boolean Valued'' in the text. In so doing we also introduce the Hasse-Diagram and relate these structures (Boolean Cube/Hasse Diagram) to characterisations of Maxinmum Concensus (MaxCon) solutions and searches over the structures just mentioned.

Next (section \ref{LP-type-subclass}) we note that the features of the MaxCon problems (and LP-Type problems in general, indeed MaxCon is just a special case) inherently dictate that we are particularly interested in a particular subclass of Boolean Monotone Functions. 

In order to make quantitative statements, and in general to reason, about the properties of MaxCon solutions, we need to define some ``ideal'' cases (section \ref{ideal-structures}). These are ``simplified'' versions of what one would intuitively characterise by terms such ``single structure'' (all the inliers are data fitting some single model, and there are no other structures, nor any ``accidental alignments of outliers'' that would be identified with small structure).

At present, we have only explored making precise statements about the relative sizes of influences.
Thus in section \ref{influence_analysis}
we take this ideal single structure and use it to derive the influences of outliers and inliers (to that structure). This enables us to prove a result that is the seed for a whole class of completely new algorithms to solve MaxCon (a class of algorithms that has at its core the estimation of influences). The key result is that the influence of outliers is greater than that of inliers. 
Further in that section we generalise to two ideal structures and then to an arbitrary ``K'' number of ideal structures in the data. 
This leads to a remarkable generalisation of the aforementioned outlier/inlier influence ordering.
In short, the inliers to larger structures have smaller influence to the inliers of smaller structures. Moreover, the influence of outliers (to all structures) has the largest influence of all. 

That section concludes with the remarkable observation that the aforementioned ordering of influences defines another boolean cube (where now the bits in the representation of a vertex define inlier to one of the ``K'' structures) {\em and another monotone function over that cube} - giving the influence of that type of data. This is a real valued monotone-decreasing function, but it is also an example of how ubiquitous monotone functions seem to be!

Since we need to {\em estimate} influences, section \ref{q-weighted} briefly discusses the notion of sampling and the related notion of ``q-weighted influences''.

Section \ref{metric-reg} returns to the ideal single structure case. It invokes a concept called ``metric regular sets'' and identifies a metric regular pairing involving the inliers to the single structure and another set of points on the cube. This concept has implications for search strategies (for MaxCon solutions) but we do not claim anything detailed or profound at this stage of research into the geometry of the search space. However, this idea, and the overall symmetries involve, do suggest that there are four basic strategies (section \ref{searches-section}) for searches that start from the top or bottom (two search strategies for each) of the Hasse-Diagram. The four are essentially made from the obvious combinations of include/exclude inlier/outlier.
Of these four, exclude outlier seem preferable because the common oracle available in MaxCon solutions not only returns feasible/infeasible but - for the infeasible case - also returns a small number of data (a basis) amongst which at least one is an outlier.

The ``theory part'' of the paper concludes (section \ref{grand_bool}) with an observation that - for the special type of Monotone Boolean Functions relevant to LP-type problems (and hence MaxCon), there is another natural Boolean Cube - which is the collection of {\em all} incarnations of MaxCon for that number of data items. Moreover, over this cube there is another monotone-decreasing integer valued function - the size of the solution of that MaxCon problem. We do not take this observation further but clearly this is the ``place/setting'' to study the properties of all MaxCon problems/solutions (as a collection).

The crucial characteristic (that some collection of sets are closed under taking subsets - in our case, feasible sets of points), is such a commonly occurring characteristic, that it is unsurprising it underlies other well known mathematical structures. 
In section \ref{matroid} we note some of these. Indeed, the notion itself goes under the name Independence Sets and is part (but only part) of the definition of a matroid. Thus, we show that MaxCon can also be seen (generally) as a problem defined on independence sets, and sometimes (but rarely) on matroids. The latter, noting the rarity, is a direct way to see why the naivie approach (greedy approach) of adding points to a set until the set becomes infeasible is almost never going to work - it will work precisely when the MaxCon problem data defines a matroid - rather than the more general independence set/BMF. While the result itself holds no surprise - no-one with any experience in MaxCon expects the greedy algorithm to work often, {\em this paper is the first paper to explicitly state when and where the greedy strategy for MaxCon will work - precisely when the associated independence set - lower part of the BMF - defines a matroid.}
Yet another (new) view of MaxCon is that it defines an abstract simplicial complex and that the MaxCon solution, itself, is the maximum sized face in this complex.

Section \ref{prop-method-section} is devoted to a first exploration of algorithms that estimate influence of data points in order to rapidly approximate the MaxCon solution. The core idea is that {\em estimates} of influence should be ordered so that (mostly) the influence of outliers to all structure should be largest and the influence of inliers to the biggest structure smallest (a result we established theoretically for the ideal cases). Of course estimates are noisy and moreover what we may consider to be outliers to all structures may have some accidental alignment (arguably a small structure): and thus the {\em estimated} influences {\em for real} data, will be a perturbation away from ideal. Nonetheless, we are able to show, using some elaborations on the basic idea, that promising results can be achieved.

\section{LP-Type Problems, Boolean Monotone Functions, MaxCon and Searches on the Boolean Cube}
\label{background-section}
LP-type problems are necessarily associated with a Monotone function (section \ref{LP-monotone}). If this function is not inherently Boolean valued, it is simple to transform the problem into a related one where a Monotone Boolean Valued function is involved - simply threshold the function of interest.

However, LP-Type problems are not necessarily formulated with reference to the Boolean Cube (though many are). 
Since our motivating problem (Maximum Consensus fitting in Computer Vision), is an LP-type problem that does naturally ``live'' on the Boolean Cube, we restrict attention to that setting in this paper.
Besides MaxCon, such a setting is already suitable for many LP-type problems, and we leave more general settings for investigation in future work.

Hence we begin with the basics of the Boolean Cube and associated Hasse-Diagram of subset inclusion,
and the concept of Monotone Boolean Functions defined over the Boolean Cube in section \ref{MBF-section}.
Of course, the Monotone Boolean Function defined over the Boolean Cube (with slight abuse of terminology we will often not distinguish between the Boolean Cube and the associated lattice or Hasse-Diagram), may represent many things, depending on the application. Again, motivated by {\em our} main application, we immediately begin talking of the function representing the feasibility or infeasiblity of some problem, but all that is really required is that the application domain property respect monotonicity (which we show holds for feasibility/infeasbility).
This immediately relates the characterisation of MaxCon (maximum sized feasible subset) as that of finding the ``highest'' zero (feasible), of a Monotone Boolean Function over the Boolean Cube.

At that point we have made clear that MaxCon is inherently connected with theory of Monotone Boolean Functions. Next we address whether we are interested in the whole class, or some clearly identified subclass. In section \ref{LP-type-subclass} we show that it is indeed a special subclass that MaxCon relates to (and for the same basic reason - related to the concept of a basis in LP-type problems - a similar statement holds more generally for any problem that can be considered as derived from an LP-type problem formulation).
The fact that our intended application relies on the properties of this special subclass, holds the promise that  for this subclass, some of the known properties of Monotone Boolean Functions (already generally more favourable than the larger class of ``just'' Boolean) may have even more favourable forms for this subclass. At the time of writing, we are not able to prove this holds, in any way directly relating to algorithms that we know of - but we are working on such issues.

We begin the detailed investigation of this sub-class in sections \ref{ideal-structures} and \ref{influence_analysis}, where we give a major contribution in showing a number of things: the influences of outlier data points are higher than those of inliers, and that the influences are in general quantised into certain values depending precisely on how the data relate to the ``shadows'' of ``upper zeros'' - essentially the structures.
This observation gives us starting ideas for a search strategy on the Boolean Cube: sections \ref{q-weighted}, \ref{metric-reg} and \ref{searches-section}, that will ultimately lead to our algorithms in section \ref{prop-method-section}.

Section \ref{grand_bool} is perhaps not so essential to those interested in our algorithmic aspects. It points out, but does not further explore, that each abstract {\em instance} of our problem (a Boolean cube of subsets of N) is the vertex in another Boolean Cube that describes all possible instances.

\subsection{Monotone Functions and LP-Type Problems}
\label{LP-monotone}
The presence of a {\em monotone} function is integral to the definition of LP-type problems as it is one of the axiomatic properties:
\begin{property}[Monotonicity]\label{prop:monotonicity}
For every sets, $P$, $Q$, $S$:  $P \subseteq Q \subseteq S$, the inequalities $f(P) \le f(Q) \le f(S)$ hold.
\end{property}


So clearly every LP-type problem is associated with a monotone function, though not necessarily one whose domain is the Boolean cube, nor whose range is also Boolean valued.
However, at least as far as the range of the function is concerned, it is a small step to turn a general Monotone function (whose range is just some totally ordered set) into a Boolean-valued Monotone function:
simply choose a threshold and now the thresholded function is Boolean and the associated LP-type problem is cast in terms of optimisation over a Monotone Boolean Function.
Indeed, this can be viewed as the way the LP-Type problem of finding the minimum width ``slab'' (region between hyperplanes) that encloses {\em all} given data points, becomes one of finding the maximum sized feasible set, feasible being defined by fitting inside a slab of fixed width (threshold) $\epsilon$.

\subsection{Monotone Boolean Function (MBF) and The Boolean Cube}
\label{MBF-section} 

\begin{figure}[ht!]
\begin{center}
\includegraphics[width=0.4\textwidth]{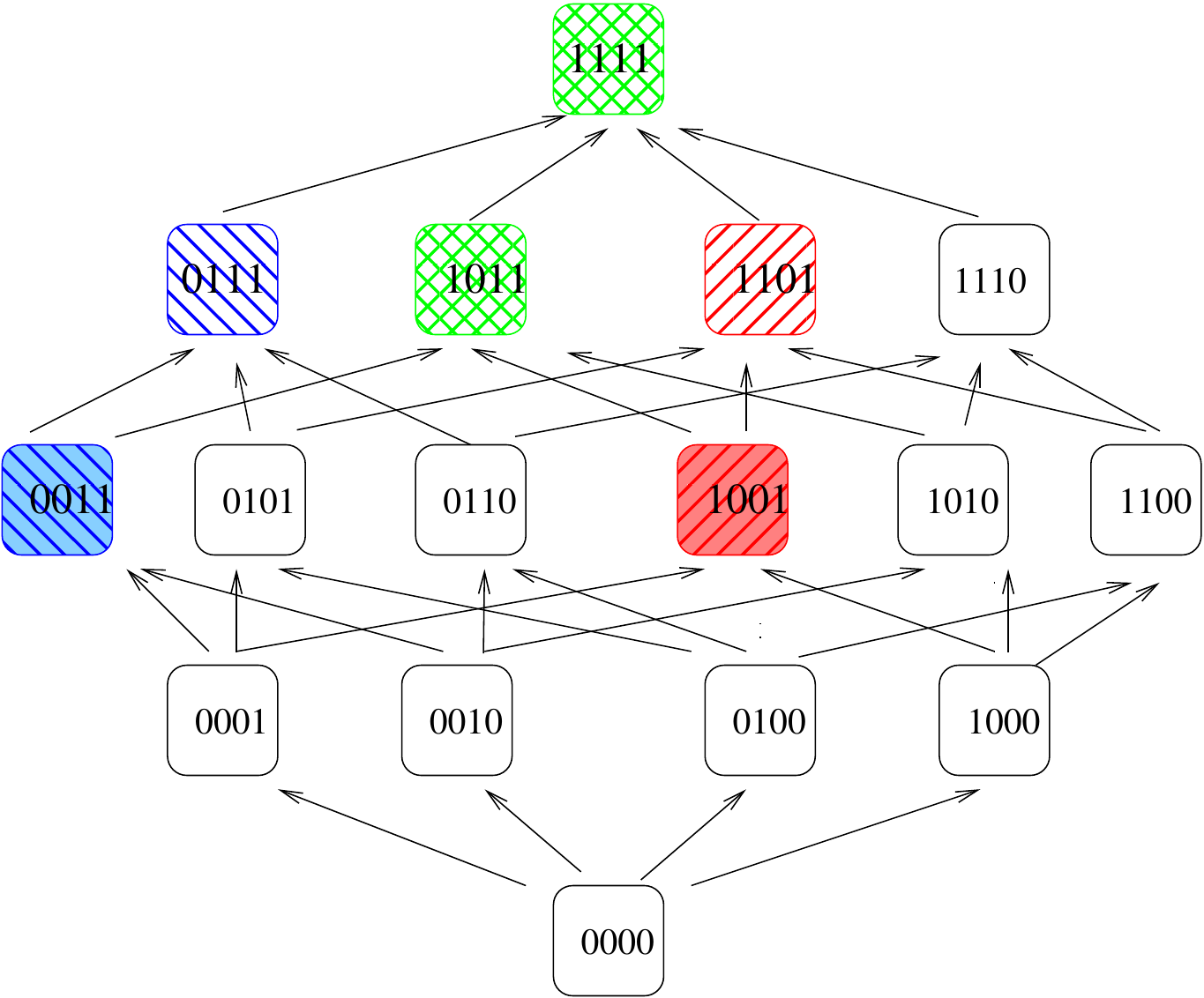}  
\end{center}

\caption{ Boolean Cube and MBF. The Boolean Cube is a generalisation of the 2D unit square (vertices (0,0), (0,1), (1,0) and (1,1)). 
The 4D Boolean Cube is drawn flattened onto 2D; and oriented so that the higher up a vertex appears, the larger is the number of 1's in the coordinates.
Interpreting the bits as coding membership in a set associated with that vertex (bit $i$ set to one if element $i$ belongs in the subset) relates the Boolean Cube to the Hasse Diagram (ordering of subsets by inclusion). In essence, the Boolean cube is this diagram ``without the arrowheads''.
A Boolean Function maps the Boolean cube to 0 or 1. We illustrate by colouring:  ``white''  nodes map to 0, coloured node (see the text below for explanation of the different colours) map to 1.
This example is a {\em Monotone} Boolean Function moving up the picture, the value of the function only ever increases, it {\em never} transitions in the opposite direction.}
\label{MBFfig}
\end{figure}

A simple example (figure \ref{MBFfig}) illustrates the concepts involved. 
For $N$ data items, the subsets can be represented by bit-strings (1 for inclusion, 0 for exclusion of data item $i$ in $i$'th position in the bit-string). Thus, each subset can be represented by a vertex of the N-dimensional Boolean Cube.
Boolean Functions (over the Boolean Cube) are just mappings to binary labels (0 and 1) that we can simply represent in figure \ref{MBFfig} with ``coloured'' (for output 1) and ``uncoloured'' (for output 0) vertices.

Associated to the concept of susbset inclusion is the Hasse Diagram: where arrows show which subsets are (minimally) included in which others. Inclusion is transitive (if subset $A$ is included in subset $B$ and $B$ is included in subset $C$ then $A$ is included in $C$) and so we say ``minimally included'' because we do not include arrows implied by transitivity. Thus the Hasse Diagram in  figure \ref{MBFfig} is the directed version of the Boolean Cube. 

Since MaxCon is a search for the maximum sized subset of the data that can be fit by the model (with given tolerance), MaxCon is inherently a search on the Boolean Cube/Hasse Diagram. Indeed, though not expressed that way in the original works, the $A^*$ ``tree searches'' of \cite{tj_2015} and \cite{Cai_2019} are in fact searches on this cube (where nodes reached by different paths become {\em repeated nodes} in the tree constructed by starting from the top of the Hasse Diagram). In this context, it is interesting to note that RANSAC (and it's many derivatives) search amongst subsets towards the bottom of the Hasse Diagram (minimal sized - $p$-sized - subsets, $p$ the minimum required for the unique determination of the model sought); and use these to ``index'' up to higher subsets by greedily including all data points the fit within tolerance of the model found on the $p$-subset.

Yet the Boolean Function that encodes whether a subset of Boolean Cube is feasible (for a given model and tolerance) is a {\em special} Boolean Function. It is a {\em Monotone} Boolean Function. That is, along paths going up the Hasse Diagram, the function can only stay constant or increase (never decrease).
This is easy to see. If a subset is feasible (can be fit by a model with given tolerance), then adding data to that subset (travelling up the Hasse Diagram) can only move towards infeasibility and - once infeasible - adding more points will not change the subset back to feasible. Likewise, deleting points from a subset can only move towards feasible (and once feasible will stay feasible under further deletion). 

One of the many special properties of Monotone Boolean Functions are that they are totally determined by their set of minimum sized subsets where the function is `1' valued. Alternatively, they are totally defined by their maximum sized subsets where the function is `0' valued. Here, by maximum and minimum sized we mean a version of ``local'' maxima/minima. For a maxima, no node/vertex immediately above, in the Hasse Diagram has a value 0; likewise, for a minima, no node immediately below has a value 1.

In figure \ref{MBFfig} the lower blue and red coloured nodes are the local minima. The whole MBF can be tabulated by simply labelling each node above (``in the upward shadow'') with 1, and by labelling every other node with zero. (In figure \ref{MBFfig} nodes coloured green are in the upward shadow of both minima). Conversely, one can identify the local maxima (there are three such nodes (1110,0110,0101) ) and label every node in their downward shadows with 0 and every other node with 1.

As an abstract representation of a MaxCon problem figure \ref{MBFfig} shows that 1110 is that MaxCon solution (as it is the highest feasible subset in the Hasse Diagram). 

Researchers have studied Monotone Boolean Functions for reasons/applications that have little to do with our application domain: (game theory \cite{Halman2007}, circuit design, complexity theory, social choice models, cryptography
- to name but a few \cite{o'donnell_2014}).

\subsection{Subclass of  Monotone Boolean Functions}
\label{LP-type-subclass}
\begin{figure}[t!]
  {  \centering
\subfigure[]{
    \includegraphics[scale=.27]{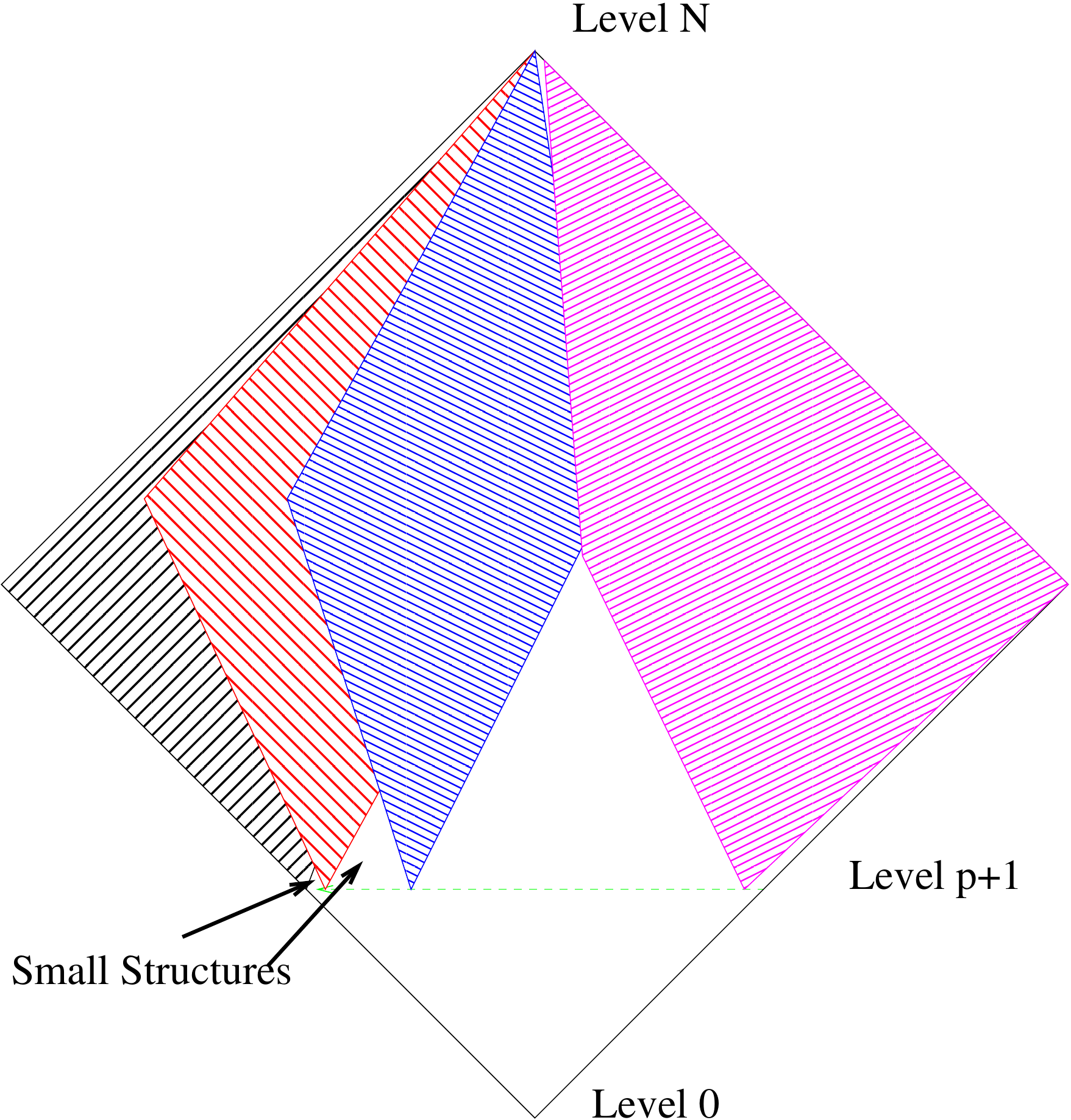}
    \label{shadow:sub-first}
}
\subfigure[]{
    \includegraphics[scale=.27]{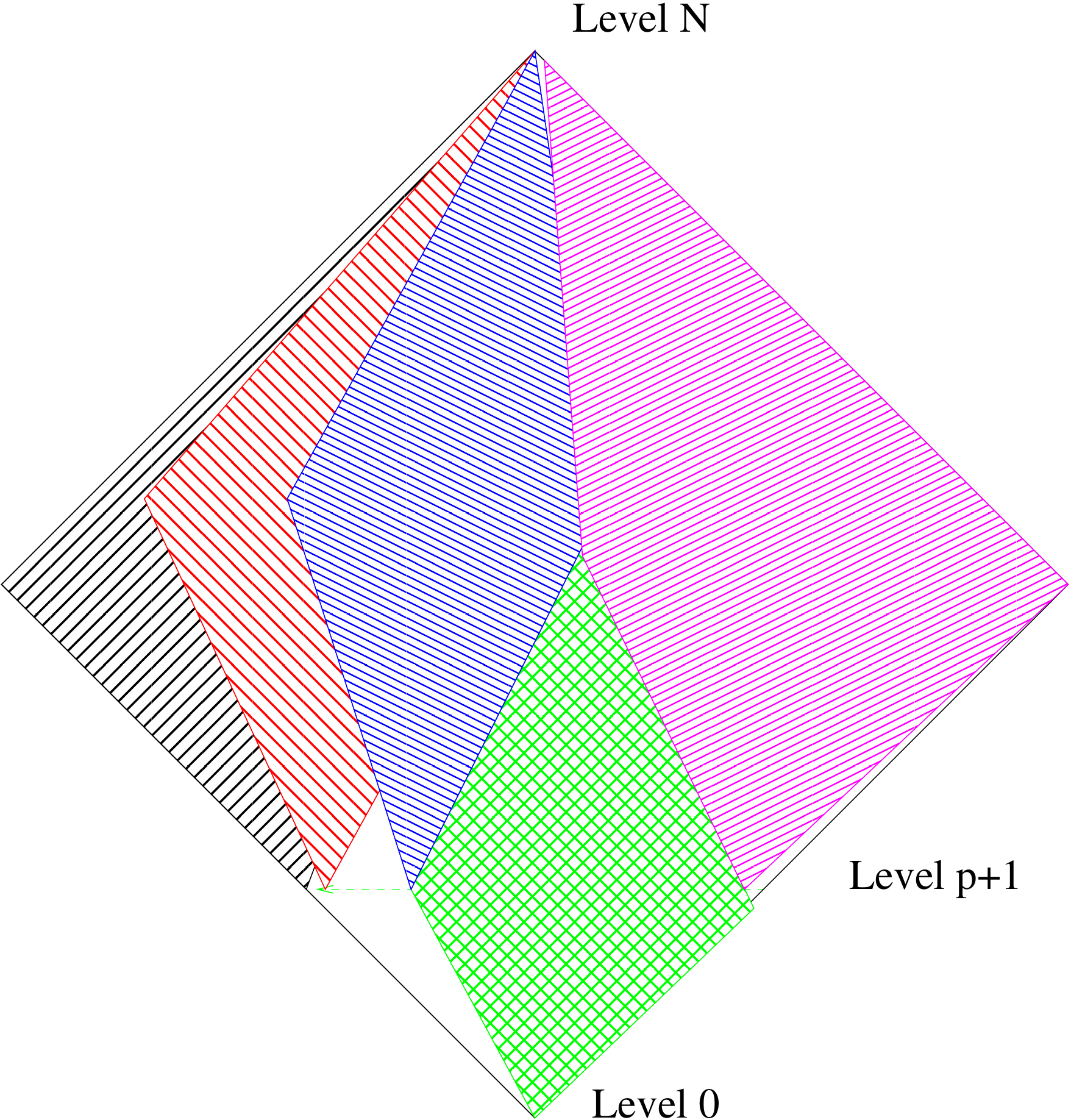} 
    \label{shadow:sub-second}
}}
\caption{Shadows of Infeasible Sets and MaxCon Solution.
(Left) Infeasible sets are defined by minimum infeasible sets {\em all of cardinality $p+1$}. All sets of cardinality $p$ and less are feasible because of the geometry of the underlying LP-type problem. Everything in the sub-cubes extending upwards from the minimum infeasible sets is infeasible (because they contain an infeasible subset) - so the feasible sets are those in the ``white region'' defined by the ``gaps'' between the infeasible regions.
(Right) There are maximum  cardinality feasible sets and these define other sub-cubes - feasible sub-cubes corresponding to ``structures'' or solutions (that can't be expended further by adding more points) - here we shade in green the one of maximum size (a MaxCon solution and all subsets of that solution).}
\end{figure}
Given that MaxCon is a special case of  LP-type formulations \cite{7873535}, we are particularly interested in a subclass of Monotone Boolean Functions: stemming from the fact that LP-type problems (and hence MaxCon) have an important property - that of ``Basis'' \cite{7873535}.
A basis, for a set is a small subset such that the value of the function on that subset is equal to the value of the function on the whole set.
Intuitively, a basis in the MaxCon setting is a subset of points ``that prevents the enclosing structure from shrinking further''.
Technically, for a given problem, the basis may be of varied cardinality up to some fixed cardinality - associated with the concept of ``combinatorial dimension'' for that problem.
This observation is the key to recognising that we are, {\em at least for these LP-type problems}, dealing with a very special case of Monotone Boolean Functions.
Those where the size of the minimum infeasible subsets (the presence of which, in a set, makes the whole set infeasible) has a known upper bound/cardinality (and usually small), see below and Fig. \ref{shadow:sub-first}.
Even more so, for many of our problems {\em all bases for that problem have the same cardinality, that is set by the combinatorial dimension for that type of problem} \cite{7873535}.
For example, in line fitting, the minimum sized subsets are all of size 3 (any two points can fit any line perfectly). This means that all of the minimum sized infeasible subsets (which collectively define the whole problem) sit at exactly the same height in the Hasse Diagram.
The above statement needs to be relaxed slightly, for cases where the bases can have varying size (up to some fixed maximum cardinality) - all minimum sized infeasible sets are at or below that fixed (maximum) cardinality.

We close this section by acknowledging that there is another simple characterisation of this class of functions. It is well known (and indeed inherent in our ``shadows'' description), that Monotone Boolean Functions can be described by Disjunctive Normal Form expressions (``OR's of AND's), formed from the minimum sized sets where the function is ``1''. (In our case, minimum infeasible sets). The function takes on a 1 exactly when it is in the shadow of one or more of these sets and so it is the ``OR'' of the requirements for being in the shadow of any of the sets. The latter is just the logical AND's of the bits in the bit positions set to one in the minimum infeasible sets. Using this characterisation, we would say that the class of functions we are interested in are those with fixed-width terms (fixed cardinality of minimum infeasible sets). Again, in cases of varying sized bases, we modify this to ``of fixed maximum width'' $p+1$ - where this is much smaller than the maximum possible width on $N$ variables (N).

\subsection{Ideal Single and Multiple Structures}
\label{ideal-structures}
It is useful, theoretically, to appeal to the ``ideal'' single structure case.
By this we mean ``no $p+1$-tuple of points is feasible {\em unless that tuple is completely contained in the one single structure''}. Of course, this rules out any bigger sized sets being feasible unless they are contained completely in that single structure. Pictorially, in Figure \ref{shadow:sub-first} we do not have any the ``secondary'' gaps between the upward shadows of infeasible subsets.

In practice, the ideal is a very unlikely situation. It implies, for example, that one never has feasible subsets that include even just one outlier to the structure. It is easy to see in the 2-dimensional line fitting case that this means that the data points have to be very well spread along the structure (so that no triple formed from 2 data points and 1 outlier is feasible (thus the spread of the data points in the to orthogonal directions, towards the outlier and along the structure, needs to be at least $\epsilon$).

\subsection{Outliers Have Larger Influence - Inliers Smaller Influence}
\label{influence_analysis}

\begin{figure}
\includegraphics[width=0.8\textwidth]{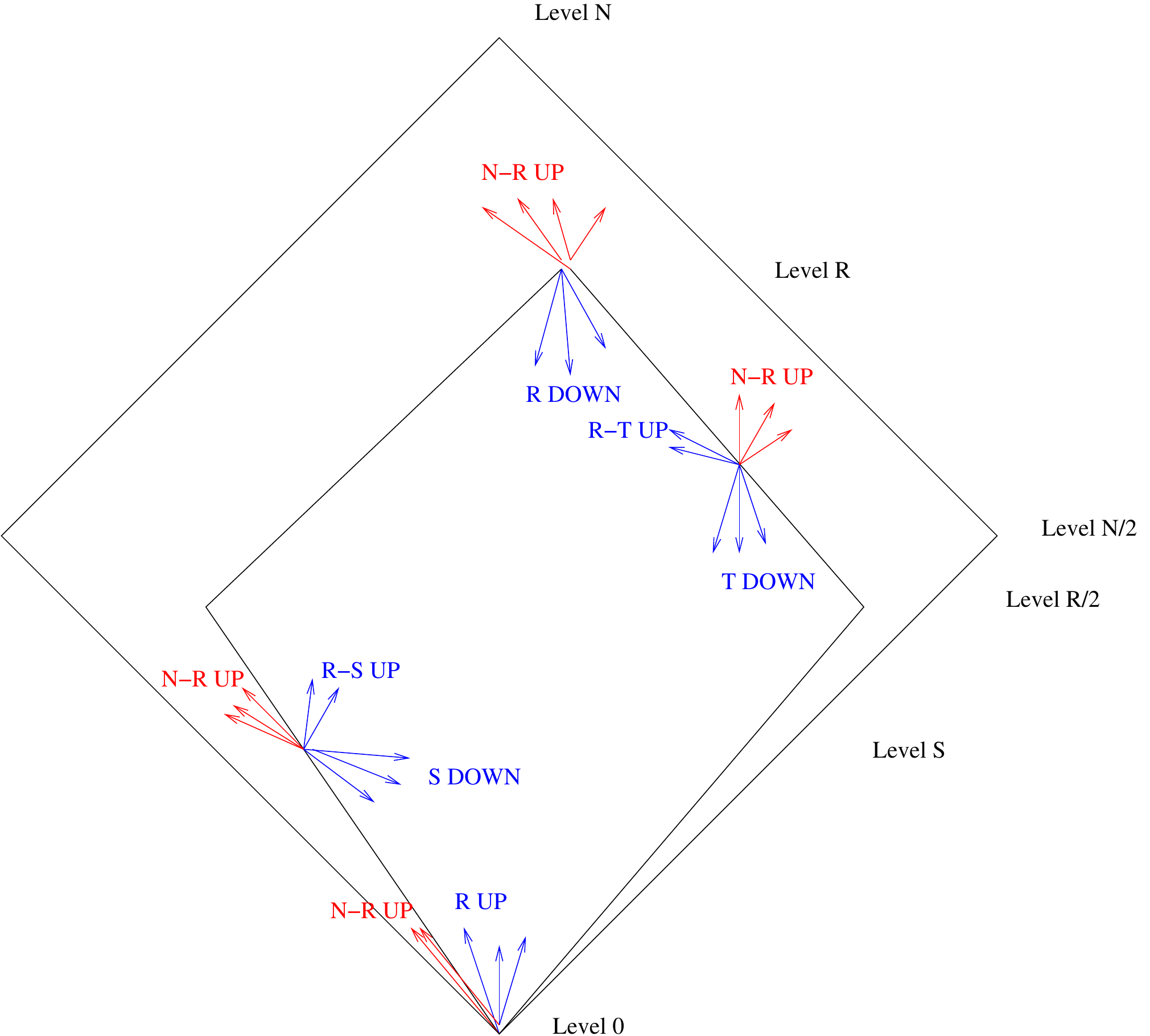}
\caption{Feasible Sub-cube}
\label{singleSub}
\end{figure}

\subsubsection{Influence}
\label{influence-section}
Thus far we have sketched new insights about the relation of MaxCon to the theory of Monotone Boolean Functions: and we have identified that we are interested in a special sub-class of such functions. But we have not indicated how the theory may be used for practical algorithms. The theory is vast and there are probably many results that suggest how to deal with MaxCon in innovative ways. Here, in this paper, we focus on only one aspect - we take the first steps towards leveraging the theory of Influence in Boolean Functions (which takes on a special form in the case of Monotone Boolean Functions).

As explained by Kalai's plenary lecture at the
2016 European Congress of Mathematics \cite{kalai:2016}, the Influence Function (I(f)) \cite{o'donnell_2014} is related to the idea of expansion in graphs (and notions akin to surface area, isoperimetric inequalities). For our purposes, the Boolean cube is divided into two sets by our Monotone Boolean  feasibility function: feasible and infeasible) and I(f) (and its components) relate to the number of edges between these sets.
Now our insight around this is that firstly, outliers to data structures will have more influence (this is {\em very} different to the notion of influence in formal robust statistics literature, and though there is a similarity in ``spirit'', the two notions of influence have vastly different character.
Here, influence refers to driving from feasible to infeasible by inclusion of an outlier in a subset; in standard robust theory it is the notion of ``dragging away an estimate of the fit'' by outliers influencing the fit more than inliers). 
Another insight, is that, in the {\em ideal} case (no secondary structures/accidental alignments etc.) all outliers will have the same (and high) influence, and all inliers will have the same (and low) influence.
Thus, in the ideal case, a very simple process of estimation of the influences, followed by thresholding, would immediately lead to the sought after MaxCon solution. Of course, the non-ideal case (which applies to all realistic data sets) will be a perturbation away from this, where the influences of inlier and outlier come closer together, and the uniformity of inlier influences, and the uniformity of outlier influences will be lost.
Nonetheless, continuity of behaviour arguments (which ought to hold in the absence of phase transitions) suggest the departure from ideal will be only partial for many data sets;
and that relatively standard ways for attacking such scenarios, should still be viable. We empirically verify that this is often true in our experiments.


\subsubsection{Influence Function in the Ideal Case}
For a Boolean Function, the total Influence is the fraction of edges in the Boolean Cube where the function changes value between the ends of that edge (\cite{o'donnell_2014} (section 2.2)). 
The total influence is contributed to by edges that correspond to flipping one particular bit (include or exclude a particular data point). Thus one can 
break the total influence up into $N$ components, where $N$ is the number of data items, to obtain the influence of each data item.
In this section, we show that for the {\em ideal} single structure case,
is such that the influences of outliers (to the single structure) is generally larger than the influence of the inlier.

\paragraph{The ``ideal single structure case'' - Individual Influences}
\label{totalsubsection}
We first formally define ideal (in a way that captures the intuitive notion of only one structure - points are either inlier or outliers to that single structure, and no other structure exists other than the trivial ($p$ or less points).

Let $L_k:=\{x\in\{0,1\}^N | \|x\|_1=k\}$ be the level $k$ in the cube $\{0,1\}^N$, $L_{\leqslant k}:=\{x\in\{0,1\}^N | \|x\|_1\leqslant k\}$ the levels below level $k+1$.

\begin{definition}
Given a monotone Boolean function $f:\{0,1\}^N\rightarrow \{0,1\}$, for $x^k\in L_k$ $(p<k\leqslant N)$, $f$ is called {\em ideal} with respect to $x^k$ if
\begin{align}
f(x)=\begin{cases}
0,~~~~&\forall~~ x\in B_{x^k}\cup L_{\leqslant p}\\
1,~~~~&\hbox{others},
\end{cases}
\end{align}
where $B_{x^k}=\{x\in\{0,1\}^N|\Delta(x,x^k)=i,~~x\in L_{k-i},~~0\leqslant i\leqslant k-p-1\}$ is the Boolean {\em sub-cube} determined by $x^k$, $\Delta(x,x^k)=\#\{i:x_i\neq x^k_i\}$ is the {\em Hamming distance} on the Boolean cube.
\end{definition}
In words, $B_{x^k}$ is the downward shadow of $x^k$.

\begin{theorem}\label{idealthm}
If a monotone Boolean function $f:\{0,1\}^N\rightarrow \{0,1\}$ is ideal with respect to $x^k\in L_k$, then
\begin{align}
2^N\hat{f}(i)&=C_p^{N-1}-C_p^{k-1},\\
2^N\hat{f}(j)&=C_p^{N-1}+\sum_{l=p+1}^k C_l^k,
\end{align}
where $i$ is an inlier and $j$ is an outlier.
\end{theorem}

\begin{proof}
At level $l$ $(p+1\leqslant l\leqslant k)$, by the definition of the Boolean sub-cube $B_{x^k}$, boundary edges (pointing from feasible area to infeasible area) only come from flipping outliers (changing from $0$ to $1$), which have $C_l^k(N-k)$ edges.

At level $p$, boundary edges are composed by three parts: $(1)$ edges pointing from $B_{x^k}\cap L_p$ to $L_{p+1}\setminus B_{x^k}$ by flipping outliers, $(2)$ edges pointing from $L_p\setminus B_{x^k}$ to $L_{p+1}\setminus B_{x^k}$ by flipping inliers, and $(3)$ edges pointing from $L_p\setminus B_{x^k}$ to $L_{p+1}\setminus B_{x^k}$ by flipping outliers.The first part has $C_p^k(N-k)$ edges. Since
\begin{align*}
2\leqslant \Delta(x,y)\leqslant \min\{2p,2(N-k)\}
\end{align*}
for any $x\in L_p\setminus B_{x^k}$ and $y\in L_p\cap B_{x^k}$, then the second part has $\sum_{l=\max\{0,p+k-N\}}^{p-1}C_l^k C_{p-l}^{N-k}(k-l)$ edges and the third part has $\sum_{l=\max\{0,p+k-N\}}^{p-1}C_l^k C_{p-l}^{N-k} (N-k-(p-l))$ edges. Since all inliers (or outliers) have the same influences, then we have
\begin{align*}
2^N\hat{f}(i)&=\sum_{l=\max\{0,p+k-N\}}^{p-1}C_l^k C_{p-l}^{N-k} (1-\frac{l}{k})\\
&=\sum_{l=\max\{0,p+k-N\}}^{p-1}(C_l^k C_{p-l}^{N-k} -C_{l-1}^{k-1} C_{p-l}^{N-k})\\
&=C_p^N-C_p^k-(C_{p-1}^{N-1}-C_{p-1}^{k-1})\\
&=C_p^{N-1}-C_p^{k-1}
\end{align*}
and
\begin{align*}
2^N\hat{f}(j)&=\sum_{l=p}^k C_l^k+\sum_{l=\max\{0,p+k-N\}}^{p-1}C_l^k C_{p-l}^{N-k} (1-\frac{p-l}{N-k})\\
&=\sum_{l=p}^k C_l^k+\sum_{l=\max\{0,p+k-N\}}^{p-1}(C_l^k C_{p-l}^{N-k}-C_l^k C_{p-l-1}^{N-k-1})\\
&=\sum_{l=p}^k C_l^k+C_p^N-C_p^k-C_{p-1}^{N-1}\\
&=C_p^{N-1}+\sum_{l=p+1}^k C_l^k.
\end{align*}
\end{proof}

\begin{corollary}
\begin{align}
2^N(\hat{f}(j)-\hat{f}(i))=\sum_{l=p+1}^k C_l^k+C_p^{k-1}>0
\end{align}
\end{corollary}

\begin{example}
\label{single-example1}
We chose $N=7$, $p=2$, $k=5$ and $x^k=1010111$. For MaxCon, this would correspond to 5 inliers (points 1,3,5,6,7) and two outliers. By direct calculation with Matlab\footnote{Codes are available after the publication of this paper.}, the scaled influences (multiply by $2^7$) are $[9~~ 31~~ 9~~ 31~~ 9~~ 9~~ 9]$. By Theorem \ref{idealthm}, we have
\begin{align*}
    &2^7\hat{f}(i)=C_2^6-C_2^4=9,\\
    &2^7\hat{f}(j)=C_2^6+\sum_{l=3}^5C_l^5=31,
\end{align*}
where $i$ is an inlier and $j$ is an outlier.
\end{example}

\paragraph{Ideal Two Structure Case}
Now we consider the ideal $2$-structure, namely, there are two upper zeros only and the Boolean sub-cubes determined by these two upper zeros disjoint above level $p$.

\begin{definition}\label{defideal2}
Let $f:\{0,1\}^N\rightarrow \{0,1\}$ be a monotone Boolean function, $x^{k_1}$ and $x^{k_2}$ are two upper zeros, where $x^{k_i}\in L_{k_i}$, $i=1,2$, $p<k_1\leqslant k_2<N$, then $f$ is called {\em ideal} with respect to $x^{k_1}$ and $x^{k_2}$, if
\begin{flalign*}
&(1)~~ \Delta(\mathcal{B}_{x^{k_1}}\setminus L_{\leqslant p},\mathcal{B}_{x^{k_2}}\setminus L_{\leqslant p})>0,~~~\hbox{and}\\
&(2)~~ f(x)=\begin{cases}
0,~~~~&\forall~~ x\in \mathcal{B}_{x^{k_1}}\cup \mathcal{B}_{x^{k_2}}\cup L_{\leqslant p},\\
1,~~~~&\hbox{others},
\end{cases}&&
\end{flalign*}
where $\Delta$ is the Hamming distance and $\mathcal{B}_{x^{k_i}}$ is the Boolean sub-cube determined by $x^{k_i}$, $i=1,2$.
\end{definition}

Let $\mathcal{S}_{x^{k_i}}^j:=\{l\in [n] | x^{k_i}_l=j\}$, $i=1,2$, $j=0,1$, we can define four types of sets
\begin{align*}
&\mathcal{S}_{11}:=S_{x^{k_1}}^1\cap S_{x^{k_2}}^1,~~~\mathcal{S}_{10}:=S_{x^{k_1}}^1\cap S_{x^{k_2}}^0,\\
&\mathcal{S}_{01}:=S_{x^{k_1}}^0\cap S_{x^{k_2}}^1,~~~\mathcal{S}_{00}:=S_{x^{k_1}}^0\cap S_{x^{k_2}}^0.
\end{align*}

In terms of consensus maximization language, $\mathcal{S}_{11}$ is the index set of points that are both inliers with respect to $x^{k_1}$ and $x^{k_2}$, $\mathcal{S}_{10}$ is the index set of points that are outliers with respect to $x^{k_1}$ while inliers with respect to $x^{k_2}$, $\mathcal{S}_{01}$ is the index set of points that are inliers with respect to $x^{k_1}$ while outliers with respect to $x^{k_2}$, $\mathcal{S}_{00}$ is the index set of points that are both outliers with respect to $x^{k_1}$ and $x^{k_2}$.

\begin{theorem}\label{idealthm2}
If a monotone Boolean function $f:\{0,1\}^N\rightarrow \{0,1\}$ is ideal with respect to $x^{k_i}\in L_{k_i}$, $i=1,2$, then
\begin{align}
2^N \hat{f}(\mathcal{S}_{11})=&C_p^{N-1}-C_p^{k_1-1}-C_p^{k_2-1},\\
2^N \hat{f}(\mathcal{S}_{10})=&C_p^{N-1}+\sum_{l=p+1}^{k_2}C_{l}^{k_2}-
C_p^{k_1-1},\\
2^N \hat{f}(\mathcal{S}_{01})=&C_p^{N-1}+\sum_{l=p+1}^{k_1}C_{l}^{k_1}-
C_p^{k_2-1},\\
2^N \hat{f}(\mathcal{S}_{00})=&C_p^{N-1}+\sum_{l=p+1}^{k_1}C_{l}^{k_1}
+\sum_{l=p+1}^{k_2}C_{l}^{k_2},
\end{align}
where $\hat{f}(\mathcal{S}_\bullet)$ denotes the influence of $j\in \mathcal{S}_\bullet$. If $\mathcal{S}_\bullet =\emptyset$, the associated influences $\hat{f}(\mathcal{S}_\bullet)$ doesn't exist.
\end{theorem}

\begin{proof}
By the first condition in Definition \ref{defideal2}, we have $0\leqslant |\mathcal{S}_{11}|\leqslant p$.

For $i=1,2$, at level $l$ $(p+1\leqslant l\leqslant k_i)$, boundary edges coming from flipping outliers with respect to $x^{k_i}$ have
\begin{align*}
\sum_{l=p+1}^{k_i}\sum_{s=\max\{0,l-(k_i-|\mathcal{S}_{11}|)\}}^{|\mathcal{S}_{11}|} C_{l-s}^{k_i-|\mathcal{S}_{11}|} C_s^{|\mathcal{S}_{11}|}=\sum_{l=p+1}^{k_i} C_l^{k_i}.
\end{align*}
At level $p$, boundary edges coming from flipping inliers with respect to both $x^{k_1}$ and $x^{k_2}$ have
\begin{align*}
&\sum_{l=0}^{|\mathcal{S}_{11} |}\sum_{h=0}^{\min \{p-l,N+|\mathcal{S}_{11} |-k_1-k_2\}}\sum_{s=\max\{0,p-h-l-(k_2-|\mathcal{S}_{11} |)\}}^{\min\{p-h-l,k_1-|\mathcal{S}_{11} |\}}C_l^{|\mathcal{S}_{11} |}C_{h}^{N+|\mathcal{S}_{11} |-k_1-k_2}C_s^{k_1-|\mathcal{S}_{11} |}C_{p-h-s-l}^{k_2-|\mathcal{S}_{11} |}(1-\frac{l}{|\mathcal{S}_{11} |})\\
&-\sum_{l=0}^{|\mathcal{S}_{11} |}C_l^{|\mathcal{S}_{11} |}(C_{p-l}^{k_1-|\mathcal{S}_{11} |}+C_{p-l}^{k_2-|\mathcal{S}_{11} |})(1-\frac{l}{|\mathcal{S}_{11} |})\\
&=C_p^N-C_{p-1}^{N-1}-(C_p^{k_1}+C_p^{k_2}-C_{p-1}^{k_1-1}-C_{p-1}^{k_2-1})\\
&=C_p^{N-1}-C_p^{k_1-1}-C_p^{k_2-1},
\end{align*}
boundary edges coming from flipping inliers with respect to $x^{k_i}$ but outliers with respect to $x^{k_j}$ $(i,j\in\{1,2\},~i\neq j)$ have
\begin{align*}
&\sum_{l=0}^{|\mathcal{S}_{11} |}\sum_{h=0}^{\min \{p-l,N+|\mathcal{S}_{11} |-k_1-k_2\}}\sum_{s=\max\{0,p-h-l-(k_i-|\mathcal{S}_{11} |)\}}^{\min\{p-h-l,k_j-|\mathcal{S}_{11} |\}}C_l^{|\mathcal{S}_{11} |}C_{h}^{N+|\mathcal{S}_{11} |-k_1-k_2}C_s^{k_1-|\mathcal{S}_{11} |}C_{p-h-s-l}^{k_2-|\mathcal{S}_{11} |}(1-\frac{p-h-s-l}{k_i-|\mathcal{S}_{11} |})\\
&-\sum_{l=0}^{|\mathcal{S}_{11} |}C_l^{|\mathcal{S}_{11} |}C_{p-l}^{k_i-|\mathcal{S}_{11} |}(1-\frac{p-l}{k_i-|\mathcal{S}_{11} |})\\
&=C_p^N-C_{p-1}^{N-1}-(C_p^{k_i}-C_{p-1}^{k_i-1})\\
&=C_p^{N-1}-C_{p}^{k_i-1}.
\end{align*}
Then, the proof completes by combining all analysis above.
\end{proof}

Based on Theorem \ref{idealthm2}, we find
\begin{corollary}
\begin{align}
&\hat{f}(\mathcal{S}_{11})<\hat{f}(\mathcal{S}_{10})\leqslant \hat{f}(\mathcal{S}_{01})<\hat{f}(\mathcal{S}_{00}),\\
&\hat{f}(\mathcal{S}_{11})+\hat{f}(\mathcal{S}_{00})=
\hat{f}(\mathcal{S}_{10})+\hat{f}(\mathcal{S}_{01}).
\end{align}
\end{corollary}

\begin{example}
\label{two-example1}
Choosing $N=9$, $p=2$, $k_1=4$, $k_2=5$, $x^{k_1}=111100000$, $x^{k_2}=001001111$,
we can `read off' the following:
there are 4 inliers to structure 1 (points 1-4), 5 inliers to structure 2 (points 3, 6-9): 1 point being a common inlier to both (point 3), and one point being an outlier to all structures (point 5).
Direct calculations by Matlab give the scaled influences as $[41~~41~~19~~ 41~~ 49~~ 27~~ 27~~ 27~~ 27]$. By Theorem \ref{idealthm2},
\begin{align*}
2^{9} \hat{f}(i_1)=&C_2^8-C_2^4-C_2^3=19, \text{inlier to both structures}\\
2^{9} \hat{f}(i_2)=&\sum_{l=3}^5C_{l}^5-C_2^3+C_2^8=41, \text{inlier to the smaller structure (structure 1) only}\\
2^{9} \hat{f}(i_3)=&\sum_{l=3}^4C_{l}^4-C_2^4+C_2^8=27, \text{inlier to the larger structure (structure 2) only}\\
2^{9} \hat{f}(i_4)=&\sum_{l=3}^5C_{l}^5+\sum_{l=3}^4C_{l}^4+C_2^8=49, \text{outlier to all structures}
\end{align*}
which are consistent with the experiments.
\end{example}

\paragraph{Ideal K-structure Case}
In what follows, we discuss the ideal $K$-structure $(K\geqslant 1)$, namely, there are $K$ upper zeros only and the Boolean sub-cubes determined by these upper zeros are disjoint above level $p$.

\begin{definition}
Let $f:\{0,1\}^N\rightarrow \{0,1\}$ be a monotone Boolean function, $x^{k_i}\in L_{k_i}$ $(1\leqslant i\leqslant K)$ are upper zeros, where $p<k_1\leqslant k_2\leqslant \cdots \leqslant k_K<N$, then $f$ is called (K-Structure) {\em $K$-ideal} with respect to $\{x^{k_i}\}_{i=1}^K$, if
\begin{flalign*}
&(1)~~ \Delta(\mathcal{B}_{x^{k_i}}\setminus L_{\leqslant p},\mathcal{B}_{x^{k_j}}\setminus L_{\leqslant p})>0,~~~\forall~~ k_i\neq k_j,~ 1\leqslant i,j\leqslant K~\hbox{and}\\
&(2)~~ f(x)=\begin{cases}
0,~~~~&\forall~~ x\in \cup_{i=1}^K \mathcal{B}_{x^{k_i}}\cup L_{\leqslant p},\\
1,~~~~&\hbox{others},
\end{cases}&&
\end{flalign*}
where $\Delta$ is the Hamming distance and $\mathcal{B}_{x^{k_i}}$ is the Boolean sub-cube determined by $x^{k_i}$, $1\leqslant i \leqslant K$.
\end{definition}
The first condition states they have no elements in common (in MaxCon they do not share data points).

Let 
$$\mathcal{S}_{x^{k_i}}^l := \{j\in [n] | x^{k_i}_j=l\},$$ 
where $l=0,1$, 
$$\mathcal{S}_{j_1 j_2\cdots j_K} := \cap_{i=1}^K \mathcal{S}_{x^{k_i}}^{j_i}, j_i\in \{0,1\}, 1\leqslant i\leqslant K.$$ 
For examples, $\mathcal{S}_{11\cdots 1}=\cap_{i=1}^K \mathcal{S}_{x^{k_i}}^1$ is the index set of points that are inliers with respect to all $x^{k_i}$, $\mathcal{S}_{00\cdots 0}=\cap_{i=1}^K \mathcal{S}_{x^{k_i}}^0$ is the index set of points that are outliers with respect to all $x^{k_i}$, $1\leqslant i\leqslant K$.

\begin{theorem}\label{idealthm3}
\begin{align}\label{infarb}
2^N\hat{f}(\mathcal{S}_{j_1 j_2\cdots j_K})=C_p^{N-1}+\sum_{j_i=0,1\leqslant i\leqslant K}\sum_{l=p+1}^{k_i}C_l^{k_i}-\sum_{j_i=1,1\leqslant i\leqslant K}C_p^{k_i-1},
\end{align}
where $\hat{f}(\mathcal{S}_{j_1 j_2\cdots j_K})$ denotes the influence $\hat{f}(j)$ of $j\in \mathcal{S}_{j_1 j_2\cdots j_K}$.
\end{theorem}

\begin{proof}
We prove the theorem by induction on $K$. When $K=1$, \eqref{infarb} holds by Theorem \ref{idealthm}.

Suppose \eqref{infarb} holds for $1,2,\cdots, K-1$, namely,
\begin{align}
2^N\hat{f}(\mathcal{S}_{j_1 j_2\cdots j_{K-1}})=C_p^{N-1}+\sum_{j_i=0,1\leqslant i\leqslant K-1}\sum_{l=p+1}^{k_i}C_l^{k_i}-\sum_{j_i=1,1\leqslant i\leqslant K-1}C_p^{k_i-1},
\end{align}
now we only have to prove
\begin{align*}
2^N\hat{f}(\mathcal{S}_{j_1 j_2\cdots j_{K}})=\begin{cases}
2^N\hat{f}(\mathcal{S}_{j_1 j_2\cdots j_{K-1}})+\sum_{l=p+1}^{k_K}C_l^{k_K},~~~~&j_K=0,\\
2^N\hat{f}(\mathcal{S}_{j_1 j_2\cdots j_{K-1}})-C_p^{k_K-1},~~~~&j_K=1.
\end{cases}
\end{align*}
Since $\mathcal{S}_{j_1\cdots j_K}=\mathcal{S}_{j_1\cdots j_{K-1}}\cap \mathcal{S}_{x^{k_K}}^{j_K}$, we consider four types of sets: $\mathcal{S}_{x^{k_K}}^{1}\setminus (\mathcal{S}_{j_1\cdots j_{K-1}}\cap \mathcal{S}_{x^{k_K}}^{1})$, $\mathcal{S}_{j_1\cdots j_{K-1}}\cap \mathcal{S}_{x^{k_K}}^{1}$, $\mathcal{S}_{j_1\cdots j_{K-1}}\cap \mathcal{S}_{x^{k_K}}^{0}$, $\mathcal{S}_{x^{k_K}}^{0}\setminus (\mathcal{S}_{j_1\cdots j_{K-1}}\cap \mathcal{S}_{x^{k_K}}^{0})$. When adding one more upper zero $x^{k_K}$, the increased boundary edges by flipping outliers with respect to $x^{k_K}$ (changing from $0$ to $1$ in $\mathcal{S}_{x^{k_K}}^0$ from level $p+1$ to $k_K$) have
\begin{align*}
\sum_{l=p+1}^{k_K}\sum_{s=\max\{0,l-(k_K-| \mathcal{S}_{j_1\cdots j_{K-1}}\cap \mathcal{S}_{x^{k_K}}^{1} |)\}}^{| \mathcal{S}_{j_1\cdots j_{K-1}}\cap \mathcal{S}_{x^{k_K}}^{1} |} C_{l-s}^{k_K-| \mathcal{S}_{j_1\cdots j_{K-1}}\cap \mathcal{S}_{x^{k_K}}^{1} |}C_s^{| \mathcal{S}_{j_1\cdots j_{K-1}}\cap \mathcal{S}_{x^{k_K}}^{1} |}=\sum_{l=p+1}^{k_K}C_l^{k_K}.
\end{align*}
While boundary edges coming from flipping inliers with respect to $x^{k_K}$ (changing from $0$ to $1$ and pointing from $L_p\setminus \mathcal{B}_{x^{k_K}}$ to $L_{p+1}\setminus \mathcal{B}_{x^{k_K}}$) decrease, which have
\begin{align*}
&\sum_{l=0}^{| \mathcal{S}_{j_1\cdots j_{K-1}}\cap \mathcal{S}_{x^{k_K}}^{1} |} C_{l}^{| \mathcal{S}_{j_1\cdots j_{K-1}}\cap \mathcal{S}_{x^{k_K}}^{1} |} C_{p-l}^{k_K-| \mathcal{S}_{j_1\cdots j_{K-1}}\cap \mathcal{S}_{x^{k_K}}^{1} |}(1-\frac{p-l}{k_K-| \mathcal{S}_{j_1\cdots j_{K-1}}\cap \mathcal{S}_{x^{k_K}}^{1} |})\\
&=C_p^{k_K}-\sum_{l=0}^{| \mathcal{S}_{j_1\cdots j_{K-1}}\cap \mathcal{S}_{x^{k_K}}^{1} |} C_{l}^{| \mathcal{S}_{j_1\cdots j_{K-1}}\cap \mathcal{S}_{x^{k_K}}^{1} |} C_{p-l-1}^{k_K-| \mathcal{S}_{j_1\cdots j_{K-1}}\cap \mathcal{S}_{x^{k_K}}^{1} |-1}\\
&=C_p^{k_K}-C_{p-1}^{k_K-1}\\
&=C_{p}^{k_K-1},
\end{align*}
which completes the proof.
\end{proof}

\begin{corollary}\label{corK}
The influences have the following order relationship
\begin{align*}
\forall~ x,y\in\{0,1\}^K,~x>y\implies \hat{f}(\mathcal{S}_{x})<\hat{f}(\mathcal{S}_{y}),
\end{align*}
which means $\hat{f}$ is a real-valued monotone decrease Boolean function. If any $ \mathcal{S}_\bullet =\emptyset$, then $\hat{f}(\mathcal{S}_\bullet)$ doesn't exist.

For any $x\in \{0,1\}^K$, suppose $x$ is in level $k$, let $\hat{\mathcal{B}}_{x}=\{y\in\{0,1\}^K | \Delta(x,y)=i, y\in L_{k+i}, 0\leqslant i \leqslant K-k\}$ be the Boolean sub-cube determined by $x$, then
\begin{align}
\hat{f}(\mathcal{S}_x)+\hat{f}(\mathcal{S}_{\bf 1})=\sum_{y\in(\hat{\mathcal{B}}_x-x-{\bf 1})}\hat{f}(\mathcal{S}_y),
\end{align}
where ${\bf 1}$ denotes the top point of the Boolean cube $\{0,1\}^K$.
\end{corollary}

\begin{example}
\label{K example1}
Choosing $N=8$, $p=2$, $k_1=3$, $k_2=4$, $k_3=5$, $x^{k_1}=10010100$, $x^{k_2}=11110000$, $x^{k_3}=01011101$, direct calculation using Matlab gives the influences as $2^{-8}[33~~13~~35~~11~~21~~19~~43~~21]$.

By Theorem \ref{idealthm3}, we get
\begin{align*}
2^8\hat{f}(\mathcal{S}_{111})&=C_2^7-C_2^2-C_2^3-C_2^4=11,\\
2^8\hat{f}(\mathcal{S}_{011})&=C_2^7+\sum_{l=3}^3C_l^3-C_2^3-C_2^4=13,\\
2^8\hat{f}(\mathcal{S}_{101})&=C_2^7-C_2^2+\sum_{l=3}^4C_l^4-C_2^4=19,\\
2^8\hat{f}(\mathcal{S}_{110})&=C_2^7-C_2^2-C_2^3+\sum_{l=3}^5C_l^5=33,\\
2^8\hat{f}(\mathcal{S}_{001})&=C_2^7+\sum_{l=3}^3C_l^3+\sum_{l=3}^4C_l^4-C_2^4=21,\\
2^8\hat{f}(\mathcal{S}_{010})&=C_2^7+\sum_{l=3}^3C_l^3-C_2^3+\sum_{l=3}^5C_l^5=35,\\
2^8\hat{f}(\mathcal{S}_{000})&=C_2^7+\sum_{l=3}^3C_l^3+\sum_{l=3}^4C_l^4+\sum_{l=3}^5C_l^5=43,
\end{align*}
where $\mathcal{S}_{100} =\emptyset$, $\hat{f}(\mathcal{S}_{100})$ doesn't exist. 

Indeed, $\hat{f}:\{0,1\}^3\rightarrow \mathbb{R}$ is a real-valued monotone decreasing Boolean function, as shown in Figure \ref{infRMBF}.

\begin{figure}
\begin{center}
\begin{tikzcd}
& \textcolor{red}\bullet \textcolor{red}\bullet \textcolor{red}\bullet(11) &\\
\bullet \textcolor{red}\bullet \textcolor{red}\bullet(13) \arrow[ur] & \textcolor{red}\bullet \bullet \textcolor{red}\bullet(19) \arrow[u] & \textcolor{red}\bullet \textcolor{red}\bullet \bullet(33) \arrow[ul]\\
\bullet \bullet \textcolor{red}\bullet(21) \arrow[u] \arrow[ur] & \bullet \textcolor{red}\bullet \bullet(35) \arrow[ul] \arrow[ur] & \textcolor{red}\bullet \bullet \bullet(\#) \arrow[ul] \arrow[u]\\
& \bullet \bullet \bullet(43) \arrow[ul] \arrow[u] \arrow[ur] &
\end{tikzcd}
\caption{Real-valued monotone decreasing Boolean function $\hat{f}$: red dot denotes $1$, black dot denotes $0$, numbers in the brackets denote the scaled influence (multiplied by $2^8$) and $\#$ denotes empty value.}\label{infRMBF}
\end{center}
\end{figure}
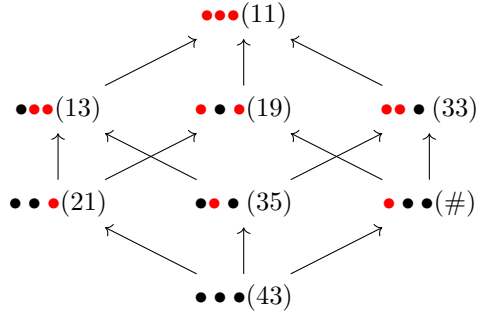

\end{example}
\subsubsection{K-Structure Non-Ideal Case}
\begin{definition}
\label{non-ideal-defn}
Let $f:\{0,1\}^N\rightarrow \{0,1\}$ be a monotone Boolean function, $x^{k_i}\in L_{k_i}$ $(1\leqslant i\leqslant K)$ are upper zeros, where $p<k_1\leqslant k_2\leqslant \cdots \leqslant k_K<N$, then $f$ is called (K-Structure) K-{\em non-ideal}, if
\begin{flalign*}
&(1)~~ \exists~ k_i\neq k_j, ~ ~(\mathcal{B}_{x^{k_i}}\setminus L_{\leqslant p})\cap(\mathcal{B}_{x^{k_j}}\setminus L_{\leqslant p})\neq \emptyset,~\hbox{and}\\
&(2)~~ f(x)=\begin{cases}
0,~~~~&\forall x\in \cup_{i=1}^K \mathcal{B}_{x^{k_i}}\cup L_{\leqslant p},\\
1,~~~~&\hbox{others},
\end{cases}&&
\end{flalign*}
where $\Delta$ is the Hamming distance and $\mathcal{B}_\bullet$ is the Boolean sub-cube determined by $\bullet$. Further,
\begin{align*}
x^{\alpha_{ij}}~:=~\underset{{x\in (\mathcal{B}_{x^{k_i}}\setminus L_{\leqslant p})\cap(\mathcal{B}_{x^{k_j}}\setminus L_{\leqslant p})}}{\mathrm{argmax}} \Vert x\Vert_1
\end{align*}
is called a {\em pseudo upper zero} with respect to $x^{k_i}$ and $x^{k_j}$. The upper index $\alpha_{ij}$ denotes the maximum $1$-norm, which is the level that the pseudo upper zero belongs to.
\end{definition}
Note: pseudo-upper zeros (and their downward shadows) are essentially of use for characterising the {\em diversity of influences} taken on by those data (in that downward shadow) compared to other inliers to the structures these pseudo-upper zeros are (in turn) inside the shadow of. In short, inliers to a given structure will have differing influences precisely when they are in different pseudo upper zeros.

Let $\{x^{\alpha_i}\}_{i=1}^M$ be the set of all pseudo upper zeros\footnote{To simplify notations, we use single sub index for $\alpha$ without indicating which two upper zeros generate the pseudo upper zero.},
\begin{align*}
\mathcal{S}_{x^{k_i}}^l := \{j\in [N] | x^{k_i}_j=l\},~~~\mathcal{S}_{x^{\alpha_i}}^l := \{j\in [N] | x^{\alpha_i}_j=1-l\},
\end{align*}
where $l=0,1$.
In words, $\mathcal{S}_{x^{k_i}}^1$ encodes (is the set of) those data that are in structure $i$ (``underneath' $x^{k_i}$),
and
$\mathcal{S}_{x^{k_i}}^0$ encodes those data that are outlier to that structure.
Likewise, $\mathcal{S}_{x^{\alpha_i}}^1$ encodes data that {\em does not} lie underneath the pseudo- upper zero $x^{\alpha_i}$ (not in the common intersection that that signifies), and $\mathcal{S}_{x^{\alpha_i}}^0$ encodes data that does.

\begin{align*}
\mathcal{S}_{j_1 j_2\cdots j_{K+M}} := (\cap_{i=1}^K \mathcal{S}_{x^{k_i}}^{j_i})\cap (\cap_{i=K+1}^{K+M} \mathcal{S}_{x^{\alpha_i}}^{j_i}), ~~j_i\in \{0,1\},~~ 1\leqslant i\leqslant K+M.
\end{align*}
In words, $\mathcal{S}_{j_1 j_2\cdots j_{K+M}}$ is the set of data in the common intersection of two collections of sets. Where the index in the first K coordinates is 1, we are finding the common inliers to those structures, and where it is a 0 the common outliers. Where the index in the last M coordinates is 1 we are finding the common data not underneath the pseudo upper zero, and where it is a 1 we are finding common data with those outside the pseudo upper zero.

For example, if $K=2$ and $M=1$, $\mathcal{S}_{000}$ is the set of points outlier to both structures (first two zeros) and also in the common intersection given by the first pseudo- upper zero. Clearly this would define an empty set in any real problem.
This example highlights that because the sets behind the coding were defined without regard to whether certain combinations will be empty sets, the Boolean Cube defined by all possible indices has vertices that we would never observe. The flip side is that we have simple encoding and ``nice symmetry''.

\begin{theorem}\label{nonidealthm}
\begin{equation}
\begin{aligned}
2^N\hat{f}(\mathcal{S}_{j_1 j_2\cdots j_{K+M}})=&C_p^{N-1}+\sum_{j_i=0,1\leqslant i\leqslant K}\sum_{l=p+1}^{k_i}C_l^{k_i}-\sum_{j_i=1,1\leqslant i\leqslant K}C_p^{k_i-1}\\
&+\sum_{j_i=0,K+1\leqslant i\leqslant K+M}C_p^{\alpha_i-1}-\sum_{j_i=1,K+1\leqslant i\leqslant K+M}\sum_{l=p+1}^{\alpha_i}C_l^{\alpha_i}
\end{aligned}
\end{equation}
where $\hat{f}(\mathcal{S}_{j_1 j_2\cdots j_{K+M}})$ denotes the influence $\hat{f}(j)$ of $j\in \mathcal{S}_{j_1 j_2\cdots j_{K+M}}$. Note $\hat{f}(\mathcal{S}_\bullet)$ doesn't exist if $\mathcal{S}_\bullet=\emptyset$.
\end{theorem}

Before proving Theorem \ref{nonidealthm}, let us consider the simplest case where $K=2,M=1$ first.
\begin{theorem}\label{lem21}
\begin{align}
2^N\hat{f}(\mathcal{S}_{111})&=C_p^{N-1}-C_p^{k_1-1}-C_p^{k_2-1}-\sum_{l=p+1}^{\alpha_1}C_l^{\alpha_1},\label{eq1}\\
2^N\hat{f}(\mathcal{S}_{110})&=C_p^{N-1}-C_p^{k_1-1}-C_p^{k_2-1}+C_p^{\alpha_1-1},\label{eq2}\\
2^N\hat{f}(\mathcal{S}_{101})&=C_p^{N-1}-C_p^{k_1-1}+\sum_{l=p+1}^{k_2}C_l^{k_2}-\sum_{l=p+1}^{\alpha_1}C_l^{\alpha_1},\label{eq3}\\
2^N\hat{f}(\mathcal{S}_{011})&=C_p^{N-1}+\sum_{l=p+1}^{k_1}C_l^{k_1}-C_p^{k_2-1}-\sum_{l=p+1}^{\alpha_1}C_l^{\alpha_1},\label{eq4}\\
2^N\hat{f}(\mathcal{S}_{100})&=C_p^{N-1}-C_p^{k_1-1}+\sum_{l=p+1}^{k_2}C_l^{k_2}+C_p^{\alpha_1-1},\label{eq5}\\
2^N\hat{f}(\mathcal{S}_{010})&=C_p^{N-1}+\sum_{l=p+1}^{k_1}C_l^{k_1}-C_p^{k_2-1}+C_p^{\alpha_1-1},\label{eq6}\\
2^N\hat{f}(\mathcal{S}_{001})&=C_p^{N-1}+\sum_{l=p+1}^{k_1}C_l^{k_1}+\sum_{l=p+1}^{k_2}C_l^{k_2}-\sum_{l=p+1}^{\alpha_1}C_l^{\alpha_1},\label{eq7}\\
2^N\hat{f}(\mathcal{S}_{000})&=C_p^{N-1}+\sum_{l=p+1}^{k_1}C_l^{k_1}+\sum_{l=p+1}^{k_2}C_l^{k_2}+C_p^{\alpha_1-1}.\label{eq8}
\end{align}
If $\mathcal{S}_\bullet=\emptyset$, $\hat{f}(\mathcal{S}_\bullet)$ doesn't exist.
\end{theorem}

\begin{proof}
By Theorem \ref{idealthm3}, we only have to prove
\begin{align*}
2^N\hat{f}(\mathcal{S}_{j_1j_2j_3})=\begin{cases}
2^N\hat{f}(\mathcal{S}_{j_1j_2})+C_p^{\alpha_1},~~~~&j_3=0,\\
2^N\hat{f}(\mathcal{S}_{j_1j_2})-\sum_{l=p+1}^{\alpha_1}C_l^{\alpha_1-1},~~~~&j_3=1.
\end{cases}
\end{align*}
This is true because in the overlap area $\mathcal{B}_{x^{\alpha_1}}$, it will kill some boundary edges from flipping $j\in\mathcal{S}_{j_1j_21}$ (which is $\sum_{l=p+1}^{\alpha_1}C_l^{\alpha_1}$) and add\footnote{Increasing is because $(\mathcal{B}_{x^{k_1}}\setminus \mathcal{B}_{x^{\alpha_1}})\cup \mathcal{B}_{x^{\alpha_1}} \cup (\mathcal{B}_{x^{k_2}}\setminus \mathcal{B}_{x^{\alpha_1}})=(\mathcal{B}_{x^{k_1}}\cup \mathcal{B}_{x^{k_2}})\setminus \mathcal{B}_{x^{\alpha_1}}$.} some boundary edges from flipping $j\in\mathcal{S}_{j_1j_20}$ (which is $C_p^{\alpha_1-1}$).
\end{proof}

Comparing all the influences, we can arrange them in a real-valued monotone decreasing Boolean function, as shown in Figure \ref{lem21RMBF}.

\begin{figure}
\begin{center}
\begin{tikzcd}
& \textcolor{red}\bullet \textcolor{red}\bullet \textcolor{red}\bullet\eqref{eq1} &\\
\bullet \textcolor{red}\bullet \textcolor{red}\bullet\eqref{eq4} \arrow[ur] & \textcolor{red}\bullet \bullet \textcolor{red}\bullet\eqref{eq3} \arrow[u] & \textcolor{red}\bullet \textcolor{red}\bullet \bullet\eqref{eq2} \arrow[ul]\\
\bullet \bullet \textcolor{red}\bullet\eqref{eq7} \arrow[u] \arrow[ur] & \bullet \textcolor{red}\bullet \bullet\eqref{eq6} \arrow[ul] \arrow[ur] & \textcolor{red}\bullet \bullet \bullet\eqref{eq5} \arrow[ul] \arrow[u]\\
& \bullet \bullet \bullet\eqref{eq8} \arrow[ul] \arrow[u] \arrow[ur] &
\end{tikzcd}
\caption{All influences in Theorem \ref{lem21} are arranged as a real-valued monotone decreasing Boolean function $\hat{f}$.}\label{lem21RMBF}
\end{center}
\end{figure}
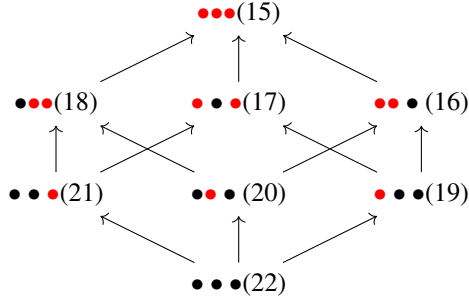

By Theorem \ref{lem21} and Theorem \ref{idealthm3}, Theorem \ref{nonidealthm} can be proved by induction on $M$.

\begin{corollary}
$\hat{f}:\{0,1\}^{K+M}\rightarrow \mathbb{R}$ is a monotone decrease Boolean function, which means
\begin{align*}
\forall x,y\in\{0,1\}^{K+M},~x>y\implies \hat{f}(\mathcal{S}_{x})<\hat{f}(\mathcal{S}_{y}),
\end{align*}
\end{corollary}

\begin{example}
\label{non-ideal-example1}
We choose $N=8$, $p=2$, $k_1=4$, $k_2=5$, $k_3=5$, $x^{k_1}=11001100$, $x^{k_2}=10101110$, $x^{k_3}=10110110$, then the pseudo upper zeros are $x^{\alpha_1}=10001100$, $x^{\alpha_2}=10100110$, which indicate $\alpha_1=3$ and $\alpha_2=4$. Direct calculations by Matlab gives the scaled influences (multiplied by $2^8$) as $[10~~ 44~~ 16~~ 30~~ 24~~ 10~~ 16~~ 52 ]$.

By Theorem \ref{nonidealthm}, we have
\begin{align*}
&2^8\hat{f}(\mathcal{S}_{11100})=C_2^7-C_2^3-C_2^4-C_2^4+C_2^2+C_2^3=10,\\
&2^8\hat{f}(\mathcal{S}_{01110})=C_2^7+\sum_{l=3}^4C_l^4-C_2^4-C_2^4-\sum_{l=3}^3C_l^3+C_2^3=16,\\
&2^8\hat{f}(\mathcal{S}_{11001})=C_2^7-C_2^3-C_2^4+\sum_{l=3}^5C_l^5+C_2^2-\sum_{l=3}^4C_l^4=24,\\
&2^8\hat{f}(\mathcal{S}_{00111})=C_2^7+\sum_{l=3}^4C_l^4+\sum_{l=3}^5C_l^5-C_2^4-\sum_{l=3}^3C_l^3-\sum_{l=3}^4C_l^4=30,\\
&2^8\hat{f}(\mathcal{S}_{10011})=C_2^7-C_2^3+\sum_{l=3}^5C_l^5+\sum_{l=3}^5C_l^5-\sum_{l=3}^3C_l^3-\sum_{l=3}^4C_l^4=44,\\
&2^8\hat{f}(\mathcal{S}_{00011})=C_2^7+\sum_{l=3}^4C_l^4+\sum_{l=3}^5C_l^5+\sum_{l=3}^5C_l^5-\sum_{l=3}^3C_l^3-\sum_{l=3}^4C_l^4=52.
\end{align*}

The relationships of influences are shown as follows
\begin{center}
\begin{tikzcd}
\textcolor{red}\bullet\textcolor{red}\bullet\textcolor{red}\bullet\bullet\bullet(10) & \bullet\textcolor{red}\bullet\textcolor{red}\bullet\textcolor{red}\bullet \bullet(16) & \textcolor{red}\bullet\textcolor{red}\bullet \bullet\bullet \textcolor{red}\bullet(24) & \bullet\bullet \textcolor{red}\bullet\textcolor{red}\bullet\textcolor{red}\bullet(30) & \textcolor{red}\bullet \bullet\bullet \textcolor{red}\bullet\textcolor{red}\bullet(44)\\
& & & & \bullet \bullet \bullet \textcolor{red}\bullet \textcolor{red}\bullet(52) \arrow[ul] \arrow[u]
\end{tikzcd}
\end{center}

Therefore, in MaxCon, guided by the sizes of the influences, we might eliminate points in the following sequence (largest first):
$$x_8\rightarrow x_2\rightarrow x_4\rightarrow x_5\rightarrow x_3,x_7\rightarrow x_1,x_6$$
until remain points are feasible. In this example, we would eliminate $x_8,x_2,x_4$, the remaining points correspond to $x^{k_2}$, which is exactly one of the two equal maximum upper zero\footnote{For multiple maximum upper zeros, choosing anyone would be fine since they are in the same level.}.

However, if we consider the actual membership of points to the structures (and their intersections)
we would arrive the Venn diagram in Figure \ref{fig:venn} and thus recognise that the structure corresponding to  $x^{k_2}$ is in a sense redundant. In that sense (ignoring the outlier data item in this example, and also ignoring any additional structure other then the three drawn in Figure \ref{realisefig:sub-first}) Figure \ref{realisefig:sub-first} could be seen as a realisation (line fitting) of this abstract MaxCon problem, and perhaps we might prefer to find the structure corresponding to  $x^{k_3}$ (the other maximum upper zero) in preference.
We must note, however, that the situation shown here is rather atypical of realisations of the real problems - usually the intersections of structures of much less data then the whole structure.

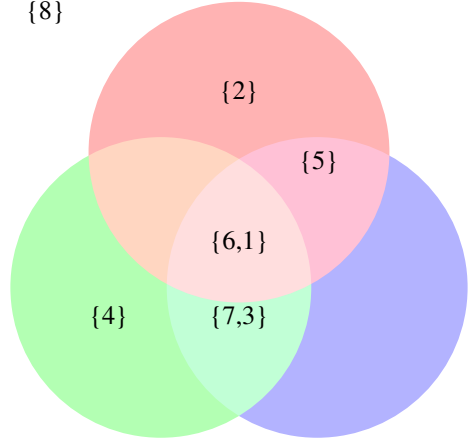
\begin{figure}
    \centering
    
\begin{tikzpicture}
\begin{scope}[blend group=soft light]
\fill[red!30!white]   ( 90:1.2) circle (2);
\fill[green!30!white] (210:1.2) circle (2);
\fill[blue!30!white]  (330:1.2) circle (2);
  \end{scope}

\node at ( 90:2)    {\{2\}};
  \node at (210:2)    {\{4\}};
  \node at (330:2)    {};
  \node  {\{6,1\}};
  \node at ( 45:1.5)    {\{5\}};
   \node at ( 270:1)    {\{7,3\}};
    \node at ( 130:4)    {\{8\}};
  \end{tikzpicture}

    \caption{Venn Diagram for Structure Membership of Example \ref{non-ideal-example1}}
    \label{fig:venn}
\end{figure}

\end{example}

\subsection{q-Weighted Influence and Influence Estimation}
\label{q-weighted}

Naively calculating the influence would involve visiting every edge in the Boolean cube (obviously infeasible for any real sized problems).

Thus one needs to estimate the influences and this leads to the topic of sampling the edges. The most obvious strategy would be to uniformly sample the cube. However, this may not be the most efficient.

This connects with the notion of weighted influence. In such an approach one typically defines a Bernoulli measure $\mu_q(x)$ over the vertices of the Boolean Cube. Operationally, this can correspond to sampling by selecting to ``turn bits on'' independently and with probability $q$. 
Uniform sampling corresponds to $q=0.5$. Sampling with $q$ low concentrates measure towards the bottom of the cube (small Hadamard norm) and high towards the top. 

One can put a slightly different interpretation on this (see for example \cite{kalai:2016}). One can view Influence as being defined in terms of this measure. Up to this point, what we have called influence ($I(f)$) can now be identified with $I_{0.5}(f)$. Hence we can now view sampling with measures $q\neq 0.5$ as either biased estimation of $I(f)$ or as estimation of $I_{q}(f)$.
In our experiments (section \ref{infl_est}) we set $q={{p+3}\over N}$ as a heuristic that seems to work better then uniform sampling for the purposes of determining likely inliers/outliers to the MaxCon solution. However, this heuristic deserves better theoretical and experimental investigation.

\subsection{Metric Regularity and Geometry of MaxCon Search}
\label{metric-reg}
Every geometric structure (e.g., our linear slab in Fig. \ref{fig:sub-second})  defines a sub-cube or face of the Boolean Cube (where $N-n$ of the coordinates - 
corresponding to the 
$N-n$ outliers to the structure - are set to zero, and the remaining coordinates can take on any values).
Faces of the Boolean Cubes actually are examples of  {\em metrically regular sets} \cite{Oblaukhov2019}. Such sets always come in pairs - and the paired set to a given structure will be also a face 
(see Fig. \ref{metric:fig}) where the $N-n$ bits are now set to one. The geometric relationship between these two sets is that they are in essence maximally far apart in the Boolean Cube. For multi-structured data, there will be a metric regular pair for each structure; but here we, for simplicity, discuss the single-structure ideal case.


It is intuitively obvious that the opposite set (in a metrically regular pair) is the region of the Boolean cube we wish to avoid (or exit early) in any search based approach to finding structures (or large feasible sets in general).
RANSAC, by it's very nature, never enters such a region. 
(The entire region consists of infeasible subsets.)
The $A^*$ tree search algorithm \cite{tjcvpr2015}, by its very nature, {\em starts} at the ``top'' of this region. The optimal path must leave that region in the first step (and spend only one step on every contour to the maximum feasible region). So the efficiency of the search relates to how many extra ``nodes'' expanded in the fringe for each contour.

The crucial point is that this metric regular pair {\em totally define the search landscape}. Contours extending out from these sets at fixed increments of {\em Hadamard size  (essentially the number bits changed)} define the gradient directions from (maximally) infeasible to feasible.

\begin{figure}
  \includegraphics[width=.5\linewidth]{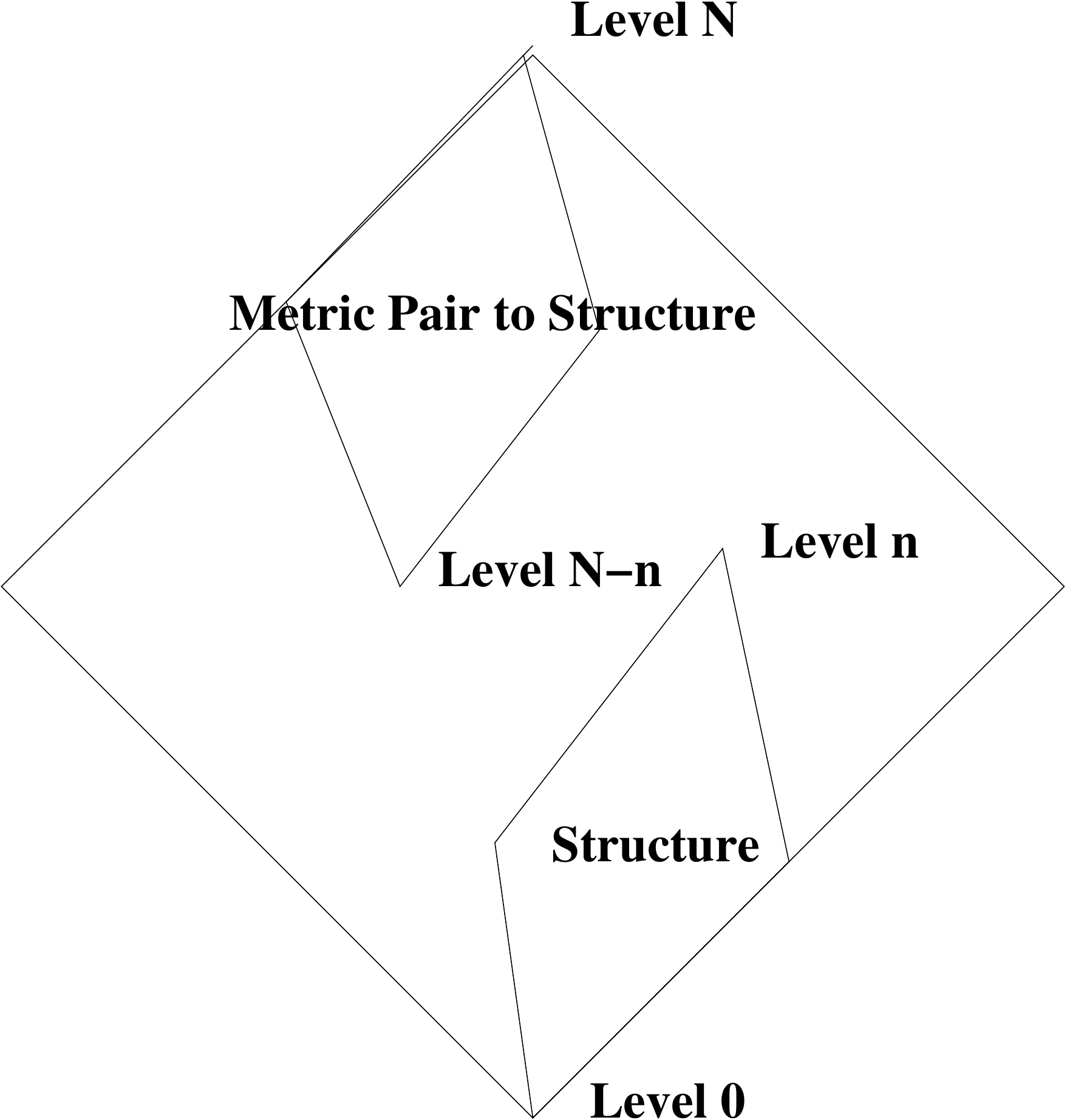}  
  \caption{Metric Regular Pairs Associated to a Structure}
\label{metric:fig}
For $N$ data, the MaxCon solution of size $n$ defines an $n$-dimensional face of the $N$ dimensional cube. The metric regular pair to this is also an $n$-dimensional face. It has ``lower'' vertex whose bits are inversion of the bits of ``upper vertex'' of the MaxCon face. In other words, the ``upper vertex'' of the MaxCon face has bits set to one for all inliers (and outlier bit set to 0) and the face is ``completed'' by removing inliers (downwards), whereas the metric regular pair face has a ``lower vertex'' with ones in all outlier positions, and the face is completed (upwards) by inclusion of inliers.

In all cases, the distance from any point in one set, to the nearest point in the other set is $N-n$. Though the diagram doesn't faithfully depict this, the closest point to $0$ (which is in the MaxCon face) is the ``lower'' vertex of the metric regular pair face. Likewise, the closest point to the top of the $N$ cube (which is in the metric regular pair face) is the MaxCon solution (the ``top'' of the MaxCon face).
Obviously, in the extreme case where the MaxCon solution is the top of the $N$ cube (all inliers), there is no longer a metric regular pair. 
\end{figure}

\subsection{Searches on the Boolean Cube}
\label{searches-section}

\begin{figure}
\includegraphics[width=.5\linewidth]{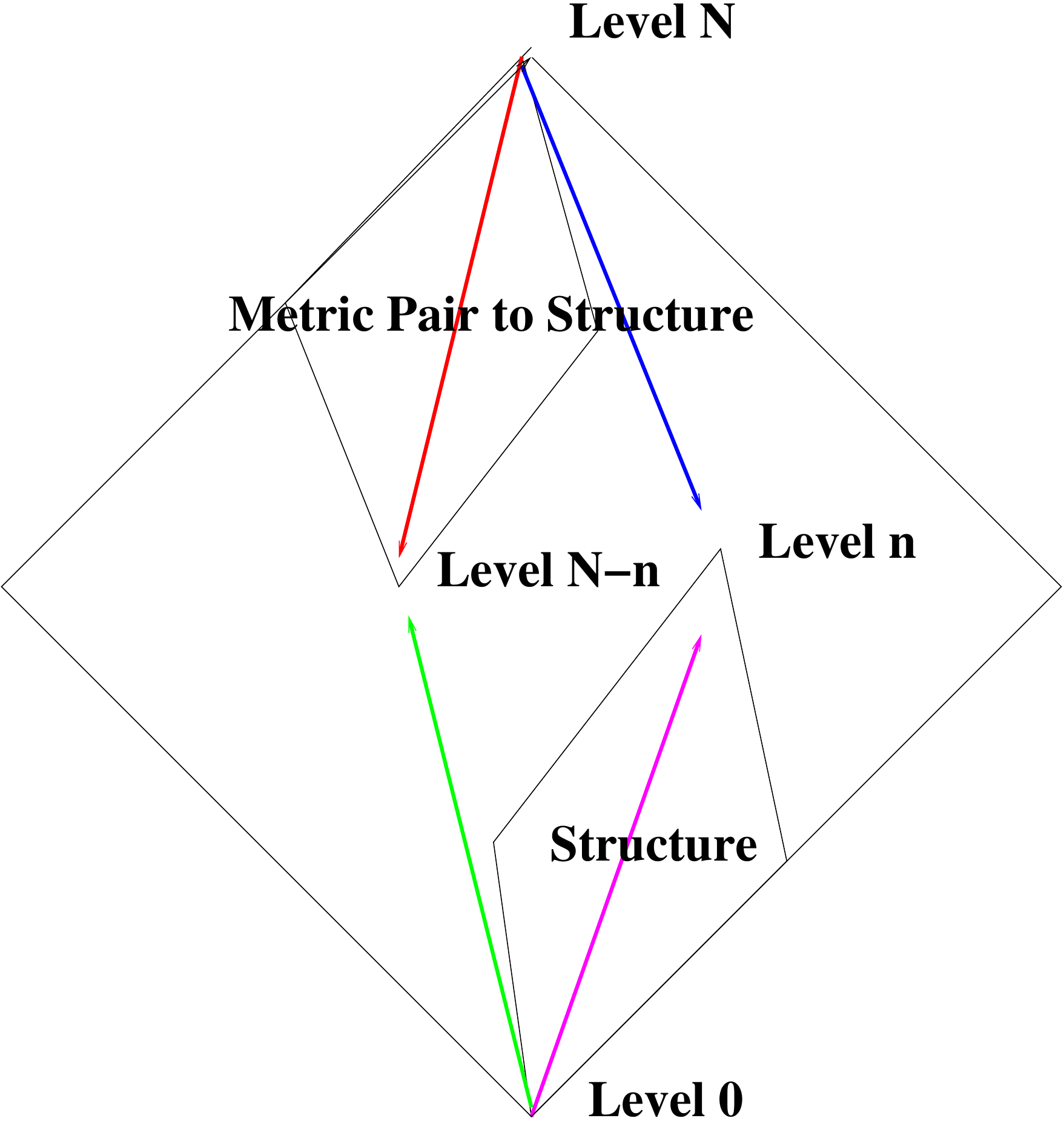}
\caption{Searches on the Boolean Cube}
\label{searches-fig}
\end{figure}

\subsubsection{Four Basic Strategies}

With the geometry of the MaxCon solution, on the Boolean Cube, in mind, a handful of search strategies seem ``natural''. 
The first is to start at the top and try to exclude outliers. This is indicated by the blue path in Figure \ref{searches-fig}. One requires an Oracle that returns feasibility/infeasibility and one stops when the first feasible subset is reached. One also requires some way of selecting the data point to eliminate at each step. $A^*$ searches in the literature \cite{tjcvpr2015} adopt this strategy and the chosen heuristic guide the order in which candidate outliers are removed.

The second is to start at the bottom and include likely inliers (the magenta path). This is not natural for the $A^*$ algorithms, at least those published so far,  because they start from the premise that the Oracle also identifies candidate points to {\em exclude} (by finding a basis). However, if one has some heuristic for determining likely inliers, this clearly is a viable search strategy. Albeit, one would also want the inliers to be ``ranked'' so as to avoid selecting inliers to sub-optimal structures. 
The analysis in section \ref{non-ideal-defn} clearly gives some indication of how ranked (estimated) influences might be employed here.

The other two strategies are perhaps less obvious but are suggested by symmetry. One can choose all four combinations (two start positions - which then dictates whether one includes or excludes) and the two choices for inliers/outliers to be included/excluded. In particular, the red path excludes likely inliers and heads for the metric regular pair vertex (all outliers). Termination can be determined by the same oracle but applied to the bit-inversion of the subset being visited. Likewise, the green path is inclusion of outliers and again the termination condition is applied to the inversion of the bits.

Of course, if the choices made (which data point to include/exclude) were based on perfect information, there would not be much to choose between these four choices. The paths are of different lengths but the complexity (with perfect information) is still essentially the same order. 
However, if the decisions are based on imperfect information (heuristics, estimates rather than exact values) then the four may behave differently and perhaps there would be reasons to combine strategies. We don't follow this suggestion here. Rather, in the experimental part of the paper we follow the same basic path strategy chosen thus far for  $A^*$ algorithms - to search from the top, looking for the blue path.
However, since our heuristic measure simultaneously gives both estimates of likely outlier (large values of our heuristic) or likely inlier (small values of our heuristic), all four strategies could be employed, in principle.

\subsubsection{Local Refinement for Sub-Optimal}

Given a feasible, but likely sub-optimal solution, the question naturally arises as to how to try to ``greedily'' improve the solution. 
Of course, many proposals exist in the literature, including the common one of using the inliers to estimate (Least Squares being. natural choice) a model fit, an then collecting all inliers within tolerance of that model fit.

Alternatively, one could cycle through all excluded points and add them to the set unless their inclusion makes the set infeasible. Such a strategy is, in general, not likely to give the optimal solution because one may add an outlier (to the optimal) early and thus be precluded from ``ascending'' to the top of the MBF shadow of the optimal upper zero. However, as section \ref{influence_analysis} shows, so long as the current set is outside of the shadow of pseudo-upper zeros (intersections of shadows of upper zeros of actual ``local optimal'' feasible sets) it should be possible to ascend to the highest upper zero. The ranking of influences of inliers outlined in section \ref{influence_analysis} (where inliers to the largest structure only - not in the overlap with other structures - are the smallest, etc.), holds some hope that such a strategy may be devised. In section \ref{refine} we implement  local improvement, which works surprisingly well, but without such sophisticated analysis. It is therefore only heuristically justified and its relative success, only empirically justified.

\subsection{All instances of the abstract problem - Another Boolean Cube (and set of monotone functions).}
\label{grand_bool}
\begin{figure}
{\centerline
\small
\setlength{\unitlength}{800sp}%
\begingroup\makeatletter\ifx\SetFigFont\undefined%
\gdef\SetFigFont#1#2#3#4#5{%
  \reset@font\fontsize{#1}{#2pt}%
  \fontfamily{#3}\fontseries{#4}\fontshape{#5}%
  \selectfont}%
\fi\endgroup%
\begin{picture}(8424,8730)(2389,-8176)
\put(6826,-3511){\makebox(0,0)[lb]{\smash{{\SetFigFont{12}{14.4}{\rmdefault}{\mddefault}{\updefault}{\color[rgb]{0,0,0}Values generally not equal at each level}%
}}}}
\thinlines
{\color[rgb]{0,0,0}\put(5476,-7036){\line( 1, 0){2250}}
}%
{\color[rgb]{0,0,0}\put(5476,-886){\line( 1, 0){2250}}
}%
{\color[rgb]{0,0,0}\put(6601,-1336){\vector( 0, 1){  0}}
\put(6601,-1336){\vector( 0,-1){5175}}
}%
\put(6901,389){\makebox(0,0)[lb]{\smash{{\SetFigFont{12}{14.4}{\rmdefault}{\mddefault}{\updefault}{\color[rgb]{0,0,0}Level $C^N_k$ - MaxCon=k-1}%
}}}}
\put(7051,-8161){\makebox(0,0)[lb]{\smash{{\SetFigFont{12}{14.4}{\rmdefault}{\mddefault}{\updefault}{\color[rgb]{0,0,0}Level 0 MaxCon=N}%
}}}}
\put(8101,-6961){\makebox(0,0)[lb]{\smash{{\SetFigFont{12}{14.4}{\rmdefault}{\mddefault}{\updefault}{\color[rgb]{0,0,0}Level 1 - MaxCon=N-1}%
}}}}
\put(8101,-886){\makebox(0,0)[lb]{\smash{{\SetFigFont{12}{14.4}{\rmdefault}{\mddefault}{\updefault}{\color[rgb]{0,0,0}Level $C^N_k-1$ - MaxCon=k}%
}}}}
{\color[rgb]{0,0,0}\put(6601,239){\line(-1,-1){4200}}
\put(2401,-3961){\line( 1,-1){4200}}
\put(6601,-8161){\line( 1, 1){4200}}
\put(10801,-3961){\line(-1, 1){4200}}
}%
\end{picture}%
}

  \caption{$C^N_k$ cube of all instances of problems with min infeasible size $k$}
  \label{metric:sub-second}
\end{figure}
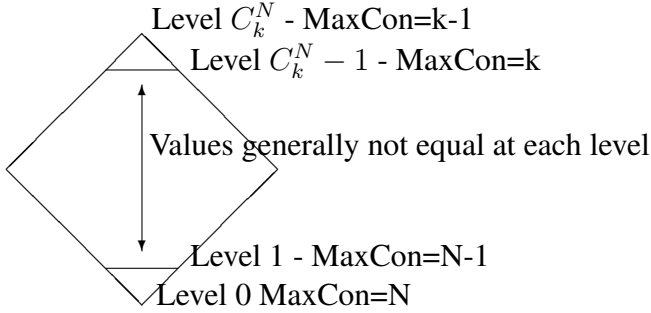

For the set of problems where infeasible sets have fixed minimum size, $k$, another boolean cube comes into play (see Fig. \ref{metric:sub-second}).
This is the Boolean Cube $\{0,1\}^{C^N_k}$.
This is a much larger Boolean Cube than $\{0,1\}^{N}$: whereas the latter describes the combinatorics of {\em an instance} of our problem, the former describes the combinatorics of {\em the collection of all problems} (of this restricted type - all minimal infeasible sets of equal size).

Specifically, $\{0,1\}^{C^N_k}$ is the set of possible choices of infeasible sets, organized into a partially ordered set (poset) on inclusion (just as the Boolean Cube $\{0,1\}^{N}$ is the set of all possible choices of data points,
in our setting, ordered by inclusion).

However, there is an added level of complication coming from the interpretation of the cube. In $\{0,1\}^{N}$ each of the $N$ elements have no ``essential relationship'' to each other. We are free to arbitrarily change the labels. In $\{0,1\}^{C^N_k}$ we have relationships between the ${C^N_k}$ elements (their ``overlap'' - where one of the infeasible subsets, of size $k$ shares elements with others).
Thus, we are not as free to ``relabel'' our constraints - relabelling must respect (preserve) these overlap relationships.

On the face of it, if $\{0,1\}^{N}$ is already a challenge to deal with (size and complexity of geometry etc.), $\{0,1\}^{C^N_k}$ would seem to be much more of a challenge. However, at least for some purposes, it has been shown (e.g., \cite{Raymond2018}) that it has characterisations that are much more efficient than would first be expected. 
It is not immediately clear whether such ideas are helpful for our purposes, but we will study the extent to which similar ideas can be used in our setting.

What are the associated monotone functions? Since every node in  $\{0,1\}^{C^N_k}$ represents an instance of our abstract feasibility problem, an obvious function is the size of the maximum feasible set for that instance (i.e., the MaxCon solution to that instance). 
At the bottom of the cube (no infeasible subsets) the function evaluates to $N$ and at the top (all k-subsets infeasible) the value of the function is $0$.
This  function is clearly monotone (since adding infeasible subsets to the problem creates a new problem instance where the MaxCon solution can only get smaller). Thresholding this function creates a new monotone function (characterising all problem instances that have a maxCon solution above/below a certain size). This is obviously a useful concept/object of study for determining conditions that ensure solutions of a certain size, and for characterising the distribution of problems of certain sizes: and perhaps the difficulty of solving sub-classes of problems (with particular configurations of minimal infeasible sets).

\subsection{Independence Systems, Matroids, Greedy Algorithm, and Simplicial Complexes}
\label{matroid}
In this section we make note of some known connections - that not connect to MaxCon.

Firstly, we note that a Boolean Monotone Function defines an Independence System. The latter is simply a ground set, and a collection of sets of the ground set closed under taking subsets. This is clearly the ``lower part'' of our Boolean Monotone Function. That is, the sets of the Boolean Cube where the Boolean Monotone Function takes on the value 0. (In our framework, the feasible sets of points form an independence system).

Now, the axioms of an an independence system as two of the axioms from a common set of axioms for the definition of a Matroid (matroids are well known to have many cryptomorphic alternative definitions). The ``missing'' third axiom (not mandated of an independence system but mandated for a matroid, is the ``exchange axiom''. In other words, a Matroid is a special case of an Independence System where one requires the matroid exchange axiom to hold. Thus, every BMF defines an Independence System (at least one - actually, by symmetry one can define an Independence System on the negated ``up-set'' as well as one the down set, the latter is the aforementioned independence set). Moreover, {\em some} Boolean Monotone Functions define a matroid (where the exchange axiom happens to hold). 

Now consider the maximum weight problem (a minimum weight problem can analogously be defined): given a set of weights on the ground set, find the maximum weight set (sum of weights of elements in a set) within the independence system. The Greedy Algorithm (start with the empty set, select from any element that can be added - leaving one still within the independence system) the one with the maximum weight (any of the choices if there is more than one). It has been known for many years that this algorithm provides the global optimal iff the independence system is a matroid. What does this mean for MaxCon? Well consider all weights equal on all of the points, then add points by the greedy algorithm, hoping to achive the maxCon solution. That is, add, at any stage, any point that when added leaves the set feasible. This naive algorithm will work iff the BMF associated with data defines a matroid (rather than any independence system that isn't also a matroid).

Note, the maximum sized sets of matroid, all have the same size. Thus, it is clear that our BMF's only ever correspond to matroids if ther is only one upper zero, or all upper zero's are equal sized. Thus, we can see that the greedy algorithm for MaxCon is essentially doomed as the aforementioned conditions are very rare in practice.

Lastly, we also note another well known connection. Independence sets are also essentially (abstract) simplicial complexes. Again, the axioms behind this concept are of a finite ground set and of closure under taking subsets. So we could alternatively view our MaxCon problem as that of finding the maximum sized face (simplex would then be the ``downsets'' of this) of an abstract simplicial complex. Put another way, the downward shadows we have been drawing, with upper zero's at their apex, are nothing other than simplices that constitute a simplicial complex, and intersecting shadows are where the simplices share faces.

\section{Exploiting Influence to Solve MaxCon}
\label{prop-method-section}
\subsection{Problem Definition and Notation}
Given a set of outlier contaminated data $\left \{\mathcal{X}_i \right\}_{i=1}^{N}$ with $N$ data points and a threshold $\epsilon$, the consensus maximization objective is to find the largest subset of data that agrees with a model instance $\theta$, up to the given threshold $\epsilon$.
We will identify the data, for notational simplicity with just their indices ${1,\ldots,N}$ and thus we can encode membership of a data item in a subset by using a binary N-bit long string, as below.

Let $\mathcal{S} =  \left \{ 0,1 \right \}^n  $ be the set of all the possible subsets of data. Each subset $s^{(j)} \in \mathcal{S}$ can be associated with a vector $\mathbf{x}^{(j)} \in \mathbb{R}^N$ where each element $x^{(j)}_i \in \mathbf{x}^{(j)}$ indicates whether the corresponding data point, $\mathcal{X}_i$, is included ($x^{(j)}_i=1$) in the subset or excluded ($x^{(j)}_i=0$). $\mathbf{x}^{(j)}$ is said to be a feasible subset if the data points, $s^{(j)} = \left \{ i : x_i^{(j)}=1 \right \}$, are predicated to be within an $\epsilon$ distance of a model instance, $\theta^{(j)}$ (infeasible otherwise). As explained in the introduction, this relationship can be modelled using a Monotone Boolean function $f : \left \{ 0,1 \right \}^N \to \left \{ 0,1 \right \}$, where $f\left( \mathbf{x} \right) = 0$ indicate the associated subset is feasible (infeasible if $1$). 

In the above setting, finding the maximum consensus set can be viewed as finding the maximum ``upper zero" of the associated MBF.  The maximum upper zero is defined as the point $\hat{\mathbf{x}}$ for which $f\left( \hat{\mathbf{x}} \right) = 0$ and for all $\mathbf{x}$ such that $\left \|  \mathbf{x}  \right \|_1 > \left \|  \hat{\mathbf{x}}  \right \|_1$, $f\left( {\mathbf{x}} \right) = 1$ \cite{KULYANOV1975267}. Note: in this setting the 1-norm is often called the Hadamard Norm and, also in this setting, it is simply a count of the number data in a subset - the cardinality of that set. 

This paper is devoted to the exploitation the information embedded in the Influence function (which is actually a vector function, one component for each data item, and coincides with the first order Fourier coefficient, $\hat{f}\left ( \left \{ i \right \} \right )$ of the MBF): to find the maximum upper zero, efficiently. To explain the method, we first note that the influence of an inlier data point is likely to be smaller than the influence of an outlier data point. The essential and intuitive reason is that inclusion of an outlier, into a feasible subset, most likely turns that subset infeasible (thereby ``influencing'' the function often). In contrast adding an inlier leaves the set feasible. More fully: inliers participate in infeasible $p+1$ sized bases with a large number of data (other outliers, subsets of inliers) whereas inliers participate in less. It is the ``breaking'' or creation of an infeasible $p+1$ sized basis that is responsible for flipping the outcome of the MBF and therefore the addition/deletion of an outlier triggers this more than that of an inlier.

In section \ref{influence_analysis} we proved the above intuition holds in certain mathematically idealized formulations. However, we need to recognise that we can only work with {\em estimated} influences, not the actual influences. Thus we need to assume/hope that the estimated influences largely follow the ordering given mathematically in our derivations and according with the above intuition. 

\subsection{First Algorithm}
For our algorithms, we use the notion of an Oracle function that will return the feasibility of a given subset of data. The actual Oracle that we employ is the same one as in \cite{tj_2015}\cite{Cai_2019} and this Oracle not only returns the required feasibility/infeasibility but it also returns a basis. (Which, as explained before, is a fixed sized set of points, contained in the given subset, that determines the feasibility of the whole subset. Hence if the subset is infeasible, at least one of these points in the basis is an outlier).

In practice our Oracle is used in several ways. Firstly, to estimate the Influence (as that process requires generation of samples from the Boolean Cube and the evaluation of the feasibility at that sample). Secondly as a termination condition (like the $A^*$ tree searches we start at the top - all data included - and terminate when we find the first feasible subset).

For our first algorithm, these are the only uses of the Oracle. (Our second algorithm uses the aforementioned extra information returned by the Oracle - since the Oracle returns a basis we can consider only evaluating the Influences of the elements corresponding to that basis - see below).

Indeed the essence of both algorithms is very simple - we calculate the Influences and remove the data point associated with the largest Influence. In the first algorithm we calculate the Influences of all data points and we call this algorithm `BMF-maxcon-max'. 

\subsection{Second Algorithm}
Since the point to be removed at each iteration must belong to the $p+1$ basis points (calculated by the Oracle, which is an $L_{\infty}$ fit to data)
we can reduce the number of degree-1 Fourier coefficients to be calculated at each iteration to $p+1$ only. This algorithm is called `BMF-maxcon-L$\infty$'. 

\subsection{More Details of Both Algorithms}
\subsubsection{Estimation of Influence}
It is important to note that we have to estimate the first order Fourier coefficients using $m$ (a hyper-parameter in our method) function queries (exact calculation using all the points in the Boolean hyper-cube is intractable for any reasonable size data set). 

In detail, our current method (we have explored others) is:
\begin{enumerate}
\item \label{est-sample} Sample $m/2$ vertices/subsets, $\left \{ \mathbf{x}^{(j)} \right \}_{j=0}^{m/2}$, randomly from Boolean cube with probability of the j'th coordinate being $p(x_j=1) =q$. In all experiments we set $q = (p+3)/N$ for sampling concentrated on level $p+3$. Though this is at the moment heuristic, it is motivated by the following: as noted before, the Monotone Boolean Functions for our setting are a special sub-class: they are fully determined by the slice at level $p+1$. Moreover, because $p$ is, again for our special cases, much less than $N$, the ``width'' of the Boolean Cube is smallest near $p$ rather than higher up. So sampling at close to this level, can be closer to exhaustive, with a given sampling budget. (As to why $p+1$ or $p+2$ as levels for concentrated sampling work less well, in our experiments, that is likely due to the fact that any samples falling at level $p$ or lower are ``wasted'' as the Boolean Function is essentially ``information free'' at those levels - always feasible.) 

\item \label{est-eval} Evaluate the function on the coordinates sampled in step \ref{est-sample}.

\item \label{est-flip} For each influence $i$: Flip the bit $i$ in all coordinates in $\left \{ \mathbf{x}^{(j)} \right \}_{j=0}^{m/2}$ and evaluate the function. The flipped coordinates are $\left \{ \mathbf{x}^{(j)} \right \}_{j=m/2+1}^{m}$. 

\item Use the calculated function value in step \ref{est-eval} and \ref{est-flip} to compute the influence.
\begin{equation}
    \hat{f}\left ( \left \{ i \right \} \right ) = \frac{1}{m} \sum_{j=1}^m f\left ( \mathbf{x}^{(j)} \right ) \cdot  (-1)^{1+x^{(j)}_i}
\end{equation}
Here $\hat{f}\left ( \left \{ i \right \} \right )$ is the influence of $i$-th data point. 
\end{enumerate}

Note: one can use the monotonic nature of the function in step \ref{est-flip} to save evaluations: where evaluation the function of the subset, before bit flip, is infeasible and the flip is $0 \to 1$, or feasible and $1 \to 0$.

\subsubsection{Refinement of Solution}
\label{refine}
The estimation of Influence introduces some noise to the proposed algorithms, and the solution returned by them may not include all the inlier points of a given structure. To partially overcome this, we introduce a local expansion step (Algorithm \ref{alg:BMF-maxcon-lexp}). In this step, starting from the initial solution, at each iteration, the distance-1 upper neighborhood (Hasse Diagram) of the current solution is explored and the current solution is updated if there is any feasible set. This process is repeated until there are no feasible subsets in the distance-1 upper neighborhood. 

The complete algorithms for both of our proposed methods are given in Algorithm \ref{alg:BMF-maxcon-linf} (with the differentiating characteristic appearing at lines 4 and 5), and the detail of local expansion refinement used in both is given in Algorithm \ref{alg:BMF-maxcon-lexp}.

\subsubsection{Re-Estimation of Influence}
\label{infl_est}
We found a further element of sophistication is required. Essentially, one needs to re-estimate the influences at each step. 
If $\hat{f}_{(l)}\left ( \left \{ i \right \} \right ) $ is the degree-1 Fourier coefficient at iteration $l$ of the algorithm, then  $\hat{f}_{(l)}\left ( \left \{ i \right \} \right ) \neq \hat{f}_{(l')}\left ( \left \{ i \right \} \right )$ when $l \neq l'$.  This is because function at level $l$ is a restricted version of the function at level $l-1$. The degree-1 Fourier coefficient at iteration $l$ of the algorithm:
\begin{equation}
\hat{f}_{(l)}\left ( \left \{ i \right \} \right ) = \hat{f}_{(l-1)}\left ( \left \{ i \right \} \right )  + \hat{f}_{(l-1)}\left ( \left \{ i, r \right \} \right ) 
\end{equation}
where $r$ is the data point, removed at iteration $l-1$. In practice re-estimating the degree-1 Fourier coefficients at each iteration is preferred as this will increase the accuracy of the estimation (restriction function progressively becomes simpler). Another allied intuition is that though, as mentioned before, noise (from both the estimation process and from the departure of the data from that of the ideal) will raise the level of the influence of some inliers (and decrease the values of some outliers) to the point where the estimated influences of some inliers will be above those of some outliers: at each stage we only remove the {\em largest} influence data point which will be away from the ``polluted'' data division; and that re-estimation afterwards allows the possibility for the re-estimated influences to be ``cleaner''.  


\begin{algorithm}[!h]                      
	\caption{Algorithm for finding the maximum consensus set using influences of BMFs.}          
	\label{alg:BMF-maxcon-linf}                           
	\begin{algorithmic} [1]
		\REQUIRE $\left \{\mathcal{X}_i \right\}_{i=1}^{N}$, $\mbox{method} \in \left \{ \mbox{`max'}, \mbox{`L}\infty\mbox{'} \right \}$, $m$.

		\STATE $\mathbf{x} \gets \left [ 1, \dots, 1 \right ]_{\left [ 1 \times N \right ]}$
		
		\REPEAT
		\STATE ${b} \gets \left \{ i : x_i = 1 \right \}$
		\IF{$\mbox{method} =  \mbox{`L}\infty\mbox{'} $}
		\STATE ${b} \gets$ The $p+1$ edge points from $L_{\infty}$ fit to $\bar{\mathcal{X}} = \left \{ \mathcal{X}_i : i \in b \right \}$
		\ENDIF
        
        \STATE Estimate $\hat{f}\left ( \left \{ i \right \} \right )~~ \forall i \in {b}$ using $m$ function queries
        \STATE $r \gets \underset{i \in {b} }{argmax}~ \hat{f}\left ( \left \{ i \right \} \right ) $
        \STATE $x_r \gets 0$
        
		\UNTIL{$f(\mathbf{x}) = 0$}
		\STATE $\mathbf{x} \gets$ Run local expansion step in (Algorithm \ref{alg:BMF-maxcon-lexp}) 
		\RETURN maximum consensus set ${s} \gets \left \{ i : x_i = 1 \right \}$
		
	\end{algorithmic}
\end{algorithm}

\begin{algorithm}[!h]                      
	\caption{Local expansion step.}          
	\label{alg:BMF-maxcon-lexp}                           
	\begin{algorithmic} [1]
		\REQUIRE $\left \{\mathcal{X}_i \right\}_{i=1}^{N}$, $m$, initial feasible set $\mathbf{x}$.

		\STATE $\bar{b} \gets \left \{ i : x_i = 0 \right \} $ 
		\REPEAT
		\STATE solFound = false
		\FORALL{$z \in \bar{b}$}
		\STATE $\bar{\mathbf{x}} \gets \mathbf{x}$;~$\bar{{x}}_z \gets 1$
		\IF{$f(\bar{\mathbf{x}}) = 0$}
		\STATE $\mathbf{x} \gets \bar{\mathbf{x}}$; solFound = true;~ break;
		\ENDIF
		\ENDFOR
		\UNTIL{solFound=true}
		\RETURN $\mathbf{x}$
		
	\end{algorithmic}
\end{algorithm}

\section{Results}
We evaluated the performance of both proposed algorithms, on both synthetic and real data, and compared those with the state-of-the-art techniques. All experiments were executed on a computer with Intel Core 2.60GHz i7 CPU, 16GB RAM and Ubuntu 14.04 OS. The publicly available codes were used to obtain the results for improved $A^*$ tree search\footnote{\url{https://github.com/ZhipengCai/MaxConTreeSearch}} (A*-NAPA-DIBP) \cite{Cai_2019} and Lo-RANSAC\footnote{\url{https://github.com/ZhipengCai/Demo---Deterministic-consensus-\\maximization-with-biconvex-programming}} \cite{loransac}. 


\subsection{Synthetic data}
To study the behaviour of the proposed algorithm under a controlled setting, similar to \cite{Cai_2019}, we conducted experiments on an 8-dimensional robust linear regression problem with synthetically generated data. First, a set of $N$ data points on a randomly instantiated model $\theta \in \mathbb{R}^8$ was generated. A subset of ($N - N_{o}$) randomly selected points (inliers) were then perturbed with uniformly distributed noise in the range $\left [ -0.1, 0.1 \right ]$. The remaining $N_{o}$ data points (outliers) were then perturbed with uniformly distributed noise from $\left [-5. -0.1\right ) \cup \left ( 0.1, 5 \right]$. The inlier threshold $\epsilon$ was therefore set to $0.1$ for all the experiments in this section.



In our experiments, the number of outliers, $N_{o}$, were varied in the range of $\left[5, 40\right]$ and the computation time for the proposed methods as well as the state-of-the-art $A^*$ method (A*-NAPA-DIBP) \cite{Cai_2019} are shown in Fig. \ref{fig:syntheticData_example_subfig1}. The figure shows that when the number of outliers are low ($< 30$) $A^*$ converges to a solution relatively quickly. However, the computational time of $A^*$ increases exponentially with the number of outliers. On the other hand, the computational time of the two proposed algorithms increase linearly with the number of outliers. This is clearly predictable as our algorithms take one step across each level and the deeper down is the MaxCon solution, proportionally longer is the ``search''. $A^*$ has a much more sophisticated search that allows backtracking of routes explored and this causes the exponential behaviour when that is heavily exercised. Figure \ref{fig:syntheticData_example_subfig2} shows the difference between the number of inliers returned by the two proposed methods and $A^*$.  On average the method BMF-maxcon-max returns a solution with usually close to the same number of inliers as the A* method, while BMF-maxcon-L$\infty$ that is around one inlier away from the optimal solution. The figure also shows the 0.05th and 0.95th percentile distances from $A^*$ solution over 100 random runs. This shows that in few cases the solution returned by  BMF-maxcon-L$\infty$ can be up to 5 inliers away from the optimal solution (around 3\% error). The main summary is that the proposed methods never ``go exponential'', unlike $A^*$, in runtime, but usually compare reasonably favourably in accuracy. 

\begin{figure}[t!]
    \centering
\subfigure[]{
    \includegraphics[scale=.27]{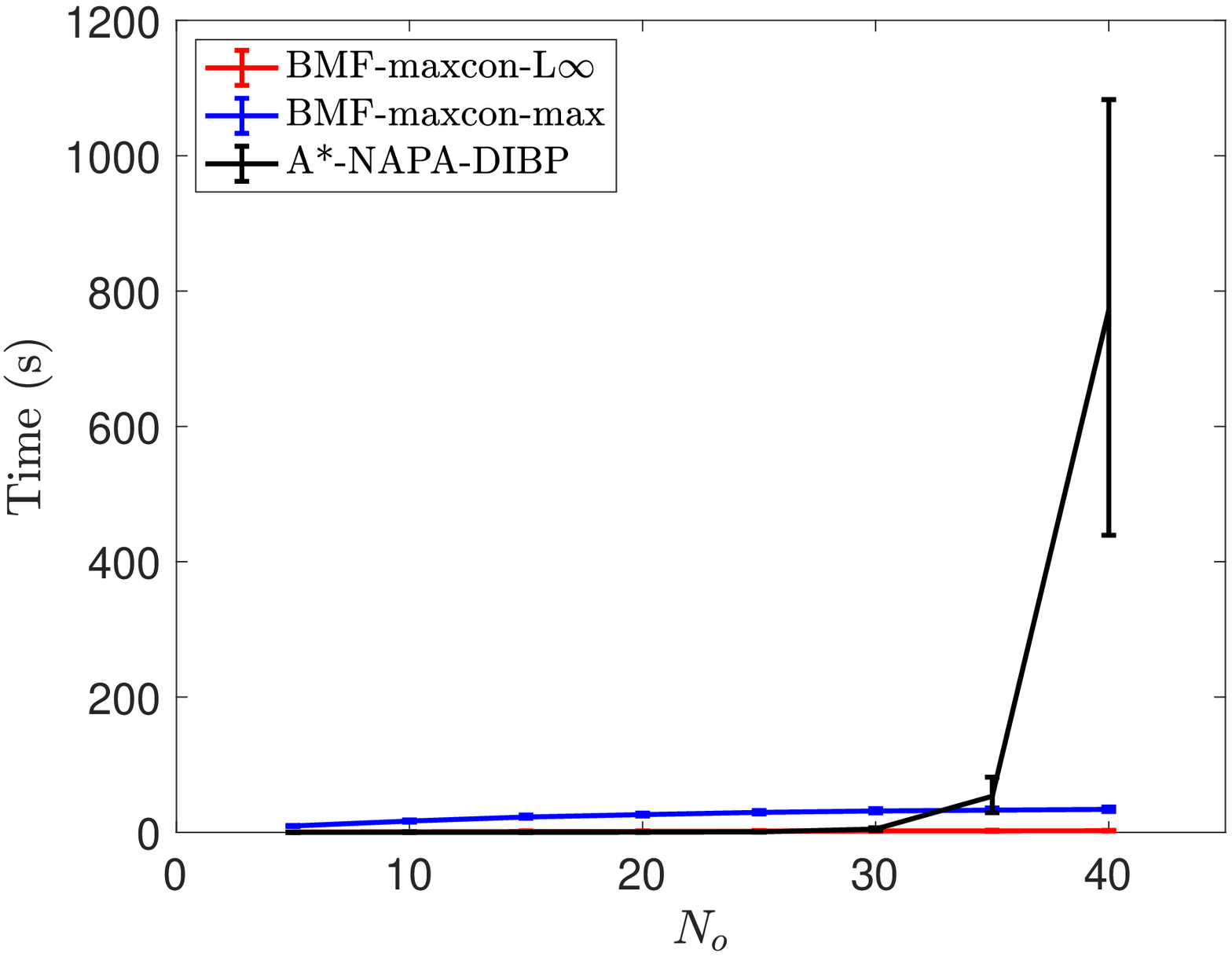} 
    \label{fig:syntheticData_example_subfig1}
}
\subfigure[]{
    \includegraphics[scale=.27]{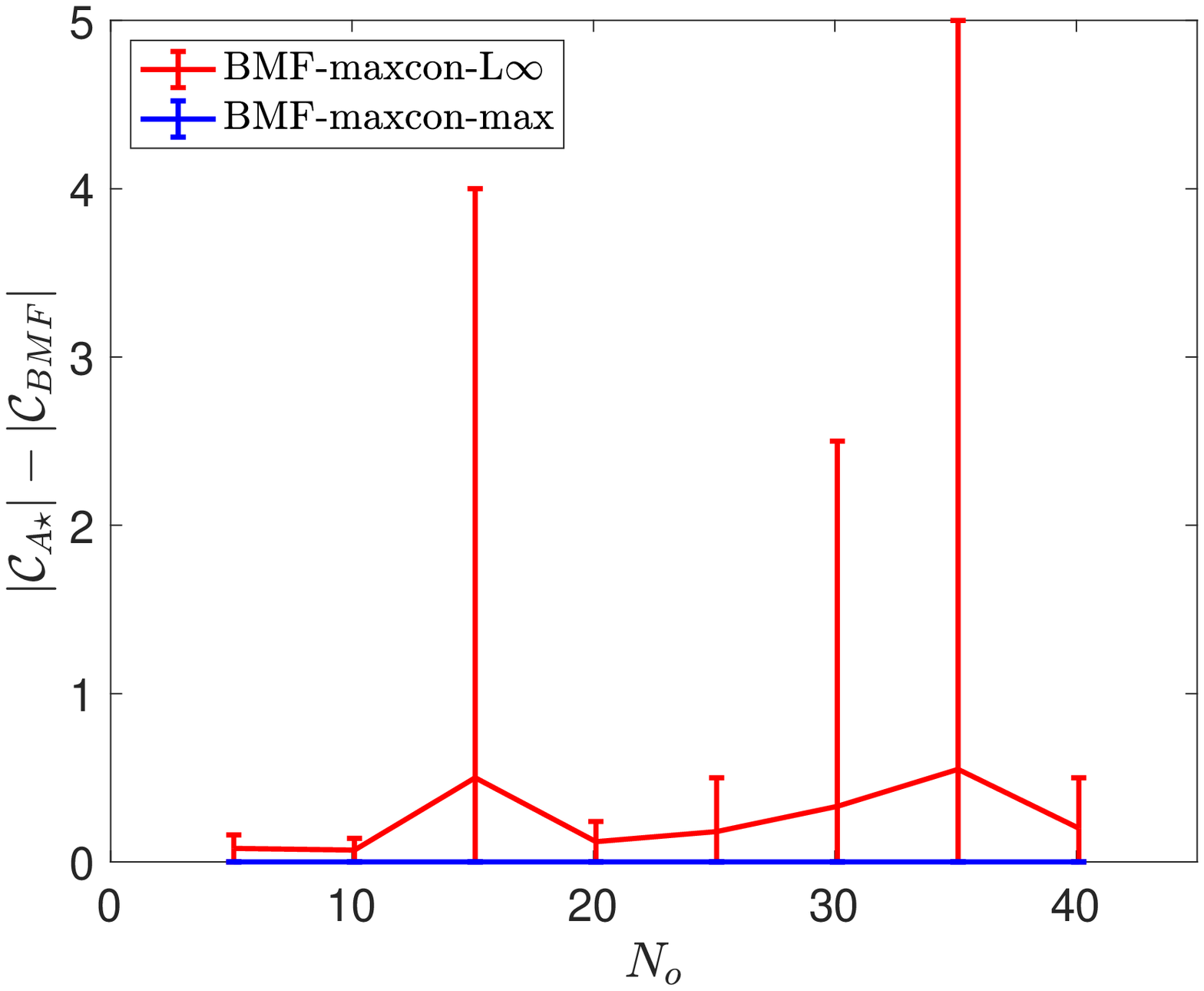} 
    \label{fig:syntheticData_example_subfig2}
}

\caption{Results for 8 dimensional robust linear regression with synthetic data (a) Variation of computational time with number of outliers and (b) Number of inliers found compared with the global optimal (obtained using $A^*$). The experiemets were repeated 100 times and the errorbars indicate the $0.05$-th and $0.95$-th percentile.}
\label{fig:syntheticData_example}
\end{figure}

\subsection{Linearized fundamental matrix estimation}
\label{FM-est}
In this section we examine the performances of the proposed methods for linearized fundamental matrix estimation \cite{tj_2015}. Provided that the point matches between two views are given as $[\textbf{x}_1, \textbf{x}_2]$ where $\textbf{x}_i = (x_i, y_i, 1)^\top$ is a coordinate of a point in view $i$, each rigid motion in the scene can be modeled using the fundamental matrix $F \in \mathcal{R}^{3 \times 3}$  as \cite{Torr1997}: $\textbf{x}_1^\top F \textbf{x}_2 = 0$. In our experiments we use the linearized version presented in \cite{tj_2015} together with the algebraic error \cite{hartley2003multiple} (chapter 11).  
\subsubsection{Motion Estimation by Fundamental Matrix Estimation}
\label{motion-est}

\paragraph{Single dominant motion:}

Following \cite{Cai_2019}, we used the first five crossroads image pairs from the sequence ``00'' of the KITTI Odometry dataset \cite{geiger2012we} in our experiments. For each image pair, the input was a set of SIFT \cite{lowe1999object} feature matches generated using VLFeat \cite{vedaldi08vlfeat}. The inlier threshold $\epsilon$ was set to $0.03$ for all image pairs. The number of inliers returned by each method ($N_i$) and the computation times are shown in Table \ref{tab:lin_fund_res}. The results reported for the proposed methods are the mean (min, max) over 100 random runs. The results show that the proposed methods on average have produced solutions that are close to the optimum solution returned by $A^*$ \cite{Cai_2019}.


{\tiny
\begin{table}[]
\centering
\caption{Linearized fundamental matrix estimation result.}
\begin{tabular}{|cl|c|c|c|c|}
\hline
\multicolumn{1}{|l}{Frame}                     &         & \begin{tabular}[c]{@{}c@{}}A*-NAPA \\ DIBP\end{tabular} & \begin{tabular}[c]{@{}c@{}}BMF-maxcon-\\ $L\infty$\end{tabular} & \begin{tabular}[c]{@{}c@{}}BMF-maxcon-\\ max\end{tabular}   & Lo-RANSAC \\ \hline
\multicolumn{1}{|c|}{\multirow{2}{*}{104-108}} & $N_{i}$ & 289          & \begin{tabular}[c]{@{}c@{}}285.10\\ (274, 289)\end{tabular}     & \begin{tabular}[c]{@{}c@{}}283.40\\ (271, 289)\end{tabular} & \begin{tabular}[c]{@{}c@{}}281.66\\ (277, 284)\end{tabular}      \\ \cline{2-6} 
\multicolumn{1}{|c|}{}                         & time (s)     & 9.627        & 1.697                                                           & 9.904                                                       &      2.21     \\ \hline
\multicolumn{1}{|c|}{\multirow{2}{*}{198-201}} & $N_{i}$ & 296          & \begin{tabular}[c]{@{}c@{}}291.50\\ (278, 296)\end{tabular}     & \begin{tabular}[c]{@{}c@{}}287.10\\ (277, 292)\end{tabular} & \begin{tabular}[c]{@{}c@{}}290.3\\ (288, 292)\end{tabular}       \\ \cline{2-6} 
\multicolumn{1}{|c|}{}                         & time (s)     & 6.047        & 1.739                                                           & 10.996                                                      &    2.22       \\ \hline
\multicolumn{1}{|c|}{\multirow{2}{*}{417-420}} & $N_{i}$ & 366          & \begin{tabular}[c]{@{}c@{}}362.15\\ (334, 366)\end{tabular}     & \begin{tabular}[c]{@{}c@{}}345.90\\ (286, 365)\end{tabular} & \begin{tabular}[c]{@{}c@{}}360.4\\ (359, 361)\end{tabular}       \\ \cline{2-6} 
\multicolumn{1}{|c|}{}                         & time (s)     & 8.037        & 2.694                                                           & 19.341                                                      &     2.28      \\ \hline
\multicolumn{1}{|c|}{\multirow{2}{*}{579-582}} & $N_{i}$ & 523          & \begin{tabular}[c]{@{}c@{}}514.00\\ (498, 523)\end{tabular}     & \begin{tabular}[c]{@{}c@{}}523.00\\ (523, 523)\end{tabular} & \begin{tabular}[c]{@{}c@{}}519.26\\ (514, 520)\end{tabular}       \\ \cline{2-6} 
\multicolumn{1}{|c|}{}                         & time (s)     & 6.942        & 4.032                                                           & 24.688                                                      &    2.26       \\ \hline
\multicolumn{1}{|c|}{\multirow{2}{*}{738-742}} & $N_{i}$ & 462          & \begin{tabular}[c]{@{}c@{}}458.50\\ (443, 462)\end{tabular}     & \begin{tabular}[c]{@{}c@{}}420.50\\ (403, 459)\end{tabular} & \begin{tabular}[c]{@{}c@{}}453.9\\ (450, 456)\end{tabular}      \\ \cline{2-6} 
\multicolumn{1}{|c|}{}                         & time (s)     & 4.151        & 1.884                                                           & 17.872                                                      &     2.25      \\ \hline
\end{tabular}
\label{tab:lin_fund_res}
\end{table}
}

The histogram of inliers returned by BMF-maxcon-L$\infty$ and Lo-RANSAC over 100 repeated runs for two frames of the KITTI Odometry dataset is shown in Figure \ref{fig:linearizedfund_hist}. The main message is that we operate in a time cost regime a little better than $A^*$ and around the same as we allowed for Lo-RANSAC but we generally get much closer to $A^*$ performance - including often finding the optimal which Lo-RANSAC never does.

\begin{figure}[t!]
    \centering
\subfigure[Frame 104-108]{
    \includegraphics[scale=.27]{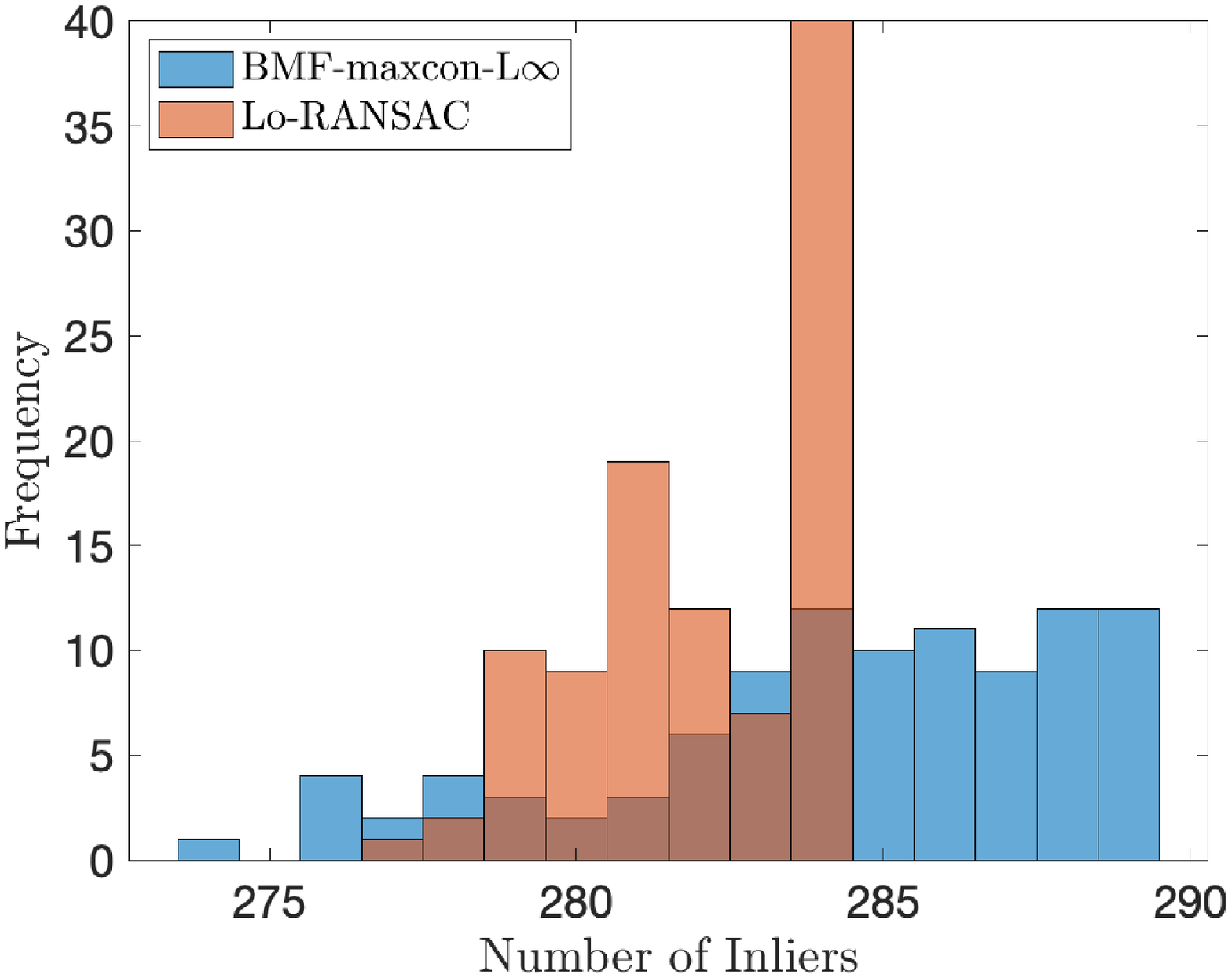} 
    \label{fig:linearizedfund_histe_subfig1}
}
\subfigure[Frame 738-742]{
    \includegraphics[scale=.27]{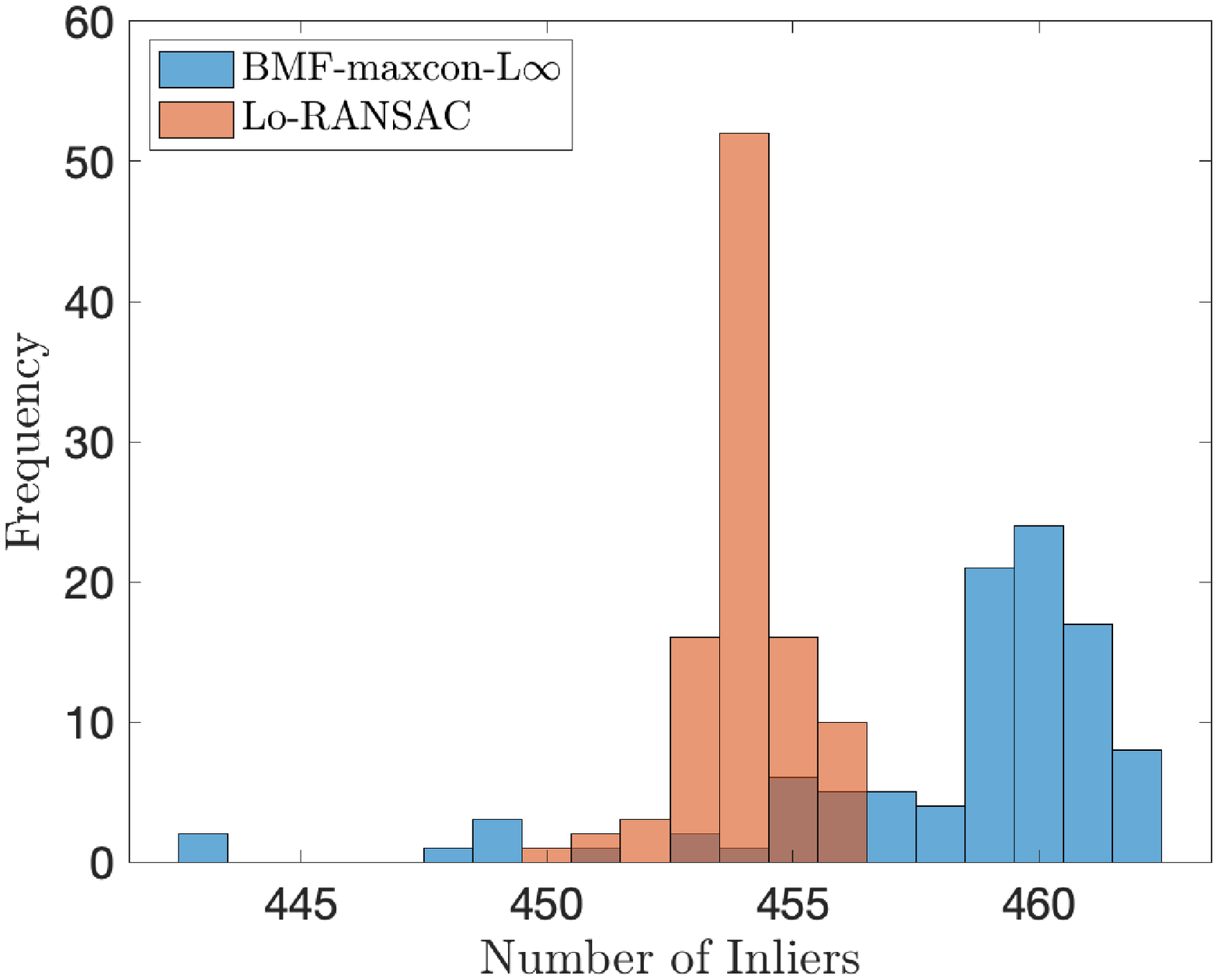} 
    \label{fig:slinearizedfund_hist_subfig2}
}

\caption{The distribution of inliers returned by BMF-maxcon-L$\infty$ and Lo-RANSAC over 100 repeated runs for frame 104-108 and frame 738-742.}
\label{fig:linearizedfund_hist}
\end{figure}

\paragraph{Multiple motions:} 
The above data-set contains a single dominant motion and the number of outliers are limited (around 13-22). To study the behaviour of the proposed algorithms in the presence of multiple structures, we conducted experiments on two sequences from the AdelaideRMF data-set \cite{wong2011dynamic}. For each image pair, the input was a set of SIFT \cite{lowe1999object} feature matches generated using VLFeat \cite{vedaldi08vlfeat}. The inlier threshold $\epsilon$ was set to $0.015$ for all image pairs. In these sequences, there are two rigidly moving objects. The results show that the method A*-NAPA-DIBP did not find a solution after $36000$s (10 hours), where as the proposed methods found solutions closer to the optimal solution in around 30 seconds. Once again, on a similar budget of computation to Lo-RANSAC, we generally obtain higher consensus sets (and thus implicitly closer to what $A^*$ is capable of, but on these data sets, would require astronomically more computation.

\begin{table}[]
\centering
\caption{Linearized fundamental matrix estimation result for cases with multiple structures.}
\begin{tabular}{|l|l|c|c|c|}
\hline
\multicolumn{2}{|l|}{}                & A*-NAPA-DIBP       & \begin{tabular}[c]{@{}c@{}}BMF-maxcon-\\ $L\infty$\end{tabular}  & Lo-RANSAC  \\ \hline
\multirow{2}{*}{breadcube} & $N_i$       & $\sim$             & \begin{tabular}[c]{@{}c@{}}63.25\\ (57, 67)\end{tabular}                       & \begin{tabular}[c]{@{}c@{}}61.5\\ (57, 65)\end{tabular}  \\ \cline{2-5} 
                           & time (s) & \textgreater 36000 & 32.9                                                                                                                            &  31.35   \\ \hline
\multirow{2}{*}{breadtoy}  & $N_i$       & $\sim$             & \begin{tabular}[c]{@{}c@{}}103.25\\ (96, 111)\end{tabular}                                                                         & \begin{tabular}[c]{@{}c@{}}103.4\\ (100, 107)\end{tabular}  \\ \cline{2-5} 
                           & time (s) & \textgreater 36000 & 29.29                                                                                                                     &  31.04   \\ \hline
\end{tabular}
\label{tab:lin_fund_res_multi}
\end{table}

\section{Conclusion}
We have applied a totally new perspective to the long standing problem of MaxCon. This perspective recognises that the underlying mathematical object is a Monotone Boolean Infeasibility function, defined over the Boolean Cube. Such a perspective immediately identifies a rich mathematical theory that can be applied. Very probably, we have only scratched that surface here. But we have been able to take at least one element of the theory - the fact that for MBFs the concept of the Influence of the function is related to the first order Fourier Transform of that function (and, which in turn can be estimated by sampling an Oracle for the function). We have linked that concept to the concept of outlier in MaxCon and shown that already, without borrowing further from the rich theory, that we can derive algorithms that are already at least competitive, in some aspects. Specifically: 
\begin{enumerate}
\item The approach sometimes achieves the true MaxCon, whereas Lo-RANSAC practically never achieves the maxCon solution. Indeed, it is well known that that is a feature of RANSAC based methods in general. It is actually provable that in most circumstances the optimal is not even available to RANSAC like sampling.
\item The approach can (mostly) achieve close to $A^*$ (provably optimal) in a similar time budget or faster.
\item Unlike $A^*$ the approach will never go exponential in runtime or memory requirements, and - as above - unlike RANSAC variants, it can more reliably obtain the optimal or close to optimal result - given a similar time budget.
\end{enumerate}

We recognise the results are thus far empirical and not so extensive as to leave these claims totally without challenge. Nonetheless, we believe it is important to demonstrate that the new theory ``has legs'': it can already claim some noteworthy performance, and exploration of our ideas, and of the wider MBF theory deserves more research attention from the computer vision community.


\bibliographystyle{ieeetr}
\bibliography{egbib,book,refs,conference}

\end{document}